\documentclass{article} 

\usepackage{iclr2022_conference,times}

\usepackage{amsmath,amsfonts,bm}

\def\eqref#1{equation~\ref{#1}}

\def\1{\bm{1}}

\DeclareMathAlphabet{\mathsfit}{\encodingdefault}{\sfdefault}{m}{sl}
\SetMathAlphabet{\mathsfit}{bold}{\encodingdefault}{\sfdefault}{bx}{n}

\newcommand{\E}{\mathbb{E}}

\newcommand{\Var}{\mathrm{Var}}

\usepackage{enumitem}
\setlist{leftmargin=5mm}

\usepackage[utf8]{inputenc} 
\usepackage[T1]{fontenc}    
\usepackage{comment}
\usepackage{url}            
\usepackage{booktabs}       
\usepackage{amsfonts}       
\usepackage{amsmath,amssymb}
\usepackage{nicefrac}       
\usepackage{microtype}      
\usepackage{enumerate,natbib,mathrsfs}
\usepackage{tablefootnote}
\usepackage{wrapfig}
\usepackage{graphicx}
\usepackage{subfigure}
\usepackage{multirow}

\usepackage{extarrows}
\usepackage[OT1]{fontenc}
\usepackage[bookmarks=false]{hyperref}

\definecolor{darkblue}{rgb}{0.0, 0.0, 0.55}
\definecolor{cornellred}{rgb}{0.7, 0.11, 0.11}
\definecolor{darkred}{rgb}{0.55, 0.0, 0.0}
\definecolor{red(ryb)}{rgb}{1.0, 0.15, 0.07}
\definecolor{red}{rgb}{1.0, 0.0, 0.0}
\hypersetup{
  pdffitwindow=true,
  pdfstartview={FitH},
  pdfnewwindow=true,
  colorlinks,
  linktocpage=true,
  linkcolor=red(ryb),
  urlcolor=red(ryb),
  citecolor=darkblue
}

\newcommand{\xeta}{x}

\newcommand{\bw}{{\boldsymbol w}}

\newcommand{\bR}{{\boldsymbol R}}

\newcommand{\bL}{{\boldsymbol L}}

\newcommand{\bU}{{\boldsymbol U}}
\newcommand{\bH}{{\boldsymbol H}}
\newcommand{\bI}{{\boldsymbol I}}

\newcommand{\MK}{{\mathcal{K}}}
\newcommand{\var}{{\mbox{Var}}}
\newcommand{\cov}{{\mbox{Cov}}}
\newcommand{\MX}{{\mathcal{X}}}

\newcommand{\btheta}{{\boldsymbol{\theta}}}

\newcommand{\bTheta}{{\boldsymbol \Theta}}

\newcommand{\bSigma}{{\boldsymbol \Sigma}}

\newcommand{\ea}{\end{array}}
\newcommand{\ee}{\end{equation}}
\newcommand{\bea}{\begin{eqnarray}}
\newcommand{\eea}{\end{eqnarray}}
\newcommand{\beaa}{\begin{eqnarray*}}
\newcommand{\eeaa}{\end{eqnarray*}}

\def\E{\mathbb{E}}

\def\bx{{\bf x}}
\def\by{{\bf y}}

\def\qed{ \hfill \vrule width.25cm height.25cm depth0cm\smallskip}

\newcommand{\basa}{\begin{assumption}}
\newcommand{\easa}{\end{assumption}}

\newcommand{\bas}{\begin{assum}}
\newcommand{\eas}{\end{assum}}

\def\limP2{\,\mathop{\buildrel \Pi_2\over\longrightarrow\,}}

\def\1{{\bf 1}}
\def\by{{\bf y}}

\def\:{\!:\!}

\newcommand{\bW}{{\boldsymbol W}}

\def\pop{\bigotimes}

\newtheorem{assump}{Assumption}

\newcommand{\bbeta}{{\boldsymbol \beta}}

\usepackage{algorithm}
\usepackage{algorithmic}

\newtheorem{remark}{Remark}
\newtheorem{theorem}{Theorem}
\newtheorem{corollary}{Corollary}
\newtheorem{lemma}{Lemma}

\newtheorem{assumption}{Assumption}

\title{Interacting Contour Stochastic Gradient
Langevin Dynamics}

\iclrfinalcopy

\author{
  Wei Deng\textsuperscript{1, 2},\;
  Siqi Liang\textsuperscript{1},\;
  Botao Hao\textsuperscript{3},\; 
  \vspace{0.3em}
  \textbf{
  Guang Lin\textsuperscript{1},\;
  Faming Liang\textsuperscript{1}\;
  }\; \\
  \vspace{0.3em}
   \textsuperscript{1}{Purdue University} 
  \textsuperscript{2}{Morgan Stanley} 
  \textsuperscript{3}{DeepMind}  \\
  \texttt{fmliang@purdue.edu; weideng056@gmail.com}
}

\begin{document}

\maketitle

\begin{abstract}
We propose an \emph{interacting} contour stochastic gradient Langevin dynamics (ICSGLD) sampler, an embarrassingly parallel multiple-chain contour stochastic gradient Langevin dynamics (CSGLD) sampler with \emph{efficient interactions}. We show  that ICSGLD can be \emph{theoretically more efficient} than a single-chain CSGLD with an equivalent computational budget. We also present a novel random-field function, which facilitates the estimation of self-adapting parameters in big data and obtains free mode explorations. Empirically, we compare the proposed algorithm with popular benchmark methods for posterior sampling. The numerical results show a great potential of ICSGLD for large-scale uncertainty estimation tasks.
\end{abstract}

\section{Introduction}
\vskip -0.05in

Stochastic gradient Langevin dynamics (SGLD) \citep{Welling11} has achieved great successes in simulations of high-dimensional systems for big data problems. It, however, yields only a fast mixing rate when the energy landscape is simple, e.g., local energy wells are shallow and not well separated. To improve its convergence for the problems with complex energy landscapes, various strategies have been proposed, such as momentum augmentation  \citep{Chen14, Ding14}, Hessian approximation \citep{Ahn12, Li16}, high-order numerical schemes \citep{Chen15, Li19}, and cyclical learning rates 
\citep{SWA1, swag, ruqi2020}. In spite of their asymptotic properties in Bayesian inference \citep{VollmerZW2016} and non-convex optimization \citep{Yuchen17}, it is still difficult to achieve compelling empirical results for  pathologically complex deep neural networks (DNNs).

 To simulate from distributions with complex energy landscapes, e.g., those with a multitude of modes well  separated by high energy barriers,
 an emerging trend is to run multiple chains, where interactions between different chains can potentially accelerate the convergence of the simulation. For example, \cite{SongWL2014} and \cite{Futoshi2020} showed theoretical advantages of appropriate interactions in ensemble/population simulations. Other multiple chain methods include particle-based nonlinear Markov (Vlasov) processes \citep{SVGD, SPOS} and replica exchange methods (also known as parallel tempering) \citep{deng_VR}. However, the particle-based methods result in an expensive kernel matrix computation given a large number of particles \citep{SVGD}; similarly, na\"{i}vely extending replica exchange methods to population chains leads to a long waiting time to swap between non-neighboring chains \citep{Doucet19}. Therefore, how to conduct interactions between different chains, while maintaining the scalability of the algorithm, is the key to the success of the parallel stochastic gradient MCMC algorithms. 

In this paper, we propose an interacting contour stochastic gradient Langevin dynamics (ICSGLD) sampler, a pleasingly parallel extension of contour stochastic gradient Langevin dynamics (CSGLD) \citep{CSGLD} with \emph{efficient interactions}. 
The proposed algorithm requires minimal communication cost in that each chain shares with others the marginal energy likelihood estimate only. 
As a result, the interacting mechanism improves the convergence of the simulation, while the minimal communication mode between different chains enables the  proposed algorithm to be naturally adapted to distributed computing with little overhead. 
For the single-chain CSGLD algorithm,  despite its theoretical advantages as shown in \cite{CSGLD}, estimation of the marginal energy likelihood remains challenging for big data problems with a wide energy range, jeopardizing the empirical performance of the class of importance sampling methods \citep{SMC1, SMC2, WangLandau2001, Liang07, Particle_MCMC, CSGLD} in big data applications. To resolve this issue, we resort to a novel interacting random-field function based on multiple chains for an ideal variance reduction and a more robust estimation. As such, we can greatly facilitate the estimation of the marginal energy likelihood so as to accelerate the simulations of notoriously complex distributions. To summarize, the algorithm has three main contributions:
\begin{itemize}
    \item We propose a scalable interacting importance sampling method for big data problems with the minimal communication cost. A novel random-field function is derived to tackle the incompatibility issue of the class of importance sampling methods in big data problems.
    \item Theoretically, we study the local stability of a non-linear mean-field system and justify regularity properties of the solution of Poisson's equation. We also prove the asymptotic normality for the stochastic approximation process in mini-batch settings and show that ICSGLD is asymptotically more efficient than the single-chain CSGLD with an equivalent computational budget. 
    \item Our proposed algorithm achieves appealing mode explorations using a fixed learning rate on the MNIST dataset and obtains remarkable performance in large-scale uncertainty estimation tasks.
\end{itemize}

\section{Preliminaries}
\vskip -0.05in
\subsection{Stochastic gradient Langevin dynamics}

A standard sampling algorithm for big data problems is SGLD \citep{Welling11}, which is a numerical scheme of a stochastic differential equation in mini-batch settings:
\begin{equation}
    \bx_{k+1}=\bx_k-\epsilon_{k} \frac{N}{n}\nabla_{\bx} \widetilde U(\bx_k)+\sqrt{2\tau \epsilon_{k}}\bw_{k},
\end{equation}
where $\bx_k\in \MX\in\mathbb{R}^{d}$, $\epsilon_{k}$ is the learning rate at iteration $k$, $N$ denotes the number of total data points,  $\tau$ is the temperature, and $\bw_k$ is a standard Gaussian vector of dimension $d$. In particular, $\frac{N}{n}\nabla_{\bx} \widetilde U(\bx)$ is an unbiased stochastic gradient estimator based on a mini-batch data $\mathcal{B}$ of size $n$ and $\frac{N}{n}\widetilde U(\bx)$ is the unbiased energy estimator for the exact energy function $U(\bx)$. Under mild conditions on $U$, $\bx_{k+1}$ is known to converge weakly to a unique invariant distribution $\pi(\bx)\propto e^{-\frac{U(\bx)}{\tau}}$ as $\epsilon_k\rightarrow 0$.

\subsection{Contour stochastic gradient Langevin dynamics}
\label{ori_csgld}
Despite its theoretical guarantees, SGLD can converge exponentially slow when $U(\bx)$ is non-convex and exhibits high energy barriers. To remedy this issue, CSGLD \citep{CSGLD} exploits the flat histogram idea and proposes to simulate from a flattened density with much lower energy barriers
\begin{equation}
\label{flat_density}
    \varpi_{\Psi_{\btheta}}(\bx) \propto {\pi(\bx)}/{\Psi^{\zeta}_{\btheta}(U(\bx))},
\end{equation}
where $\zeta$ is a hyperparameter, $\Psi_{\btheta}(u)= \sum_{i=1}^m \bigg(\theta(i-1)e^{(\log\theta(i)-\log\theta(i-1)) \frac{u-u_{i-1}}{\Delta u}}\bigg) 1_{u_{i-1} < u \leq u_i}.$

In particular, $\{u_i\}_{i=0}^m$ determines the partition $\{\MX_i\}_{i=1}^m$ of $\MX$ such that $\MX_i=\{\bx: u_{i-1}<U(\bx)\leq u_i\}$, where $-\infty=u_0<u_1<\cdots<u_{m-1}<u_m=\infty$. For practical purposes, we assume $u_{i+1}-u_i=\Delta u$ for $i=1,\cdots, m-2$. In addition, $\btheta=(\theta(1), \theta(2), \ldots, \theta(m))$ 
 is the self-adapting parameter in the space $\bTheta=\bigg\{\left(\theta(1),\cdots, \theta(m)\right)\big|0<\theta(1),\cdots,\theta(m)<1 \&\ \sum_{i=1}^m \theta(i)=1 \bigg\}$.
 
Ideally, setting $\zeta=1$ and $\theta(i)=\theta_{\infty}(i)$, where $\theta_{\infty}(i)=\int_{\MX_i} \pi(\bx)d\bx$ for $i\in\{1,2,\cdots, m\}$, enables CSGLD to achieve a ``random walk'' in the space of energy and to penalize the over-visited partition \citep{WangLandau2001, Liang07, Fortetal2011, Fort15}. However, the optimal values of $\btheta_{\infty}$ is unknown {\it a priori}. To tackle this issue, CSGLD proposes the following procedure to adaptively estimate $\btheta$ via stochastic approximation (SA) \citep{RobbinsM1951, Albert90}:
\begin{itemize}
\item[(1)] Sample $\bx_{k+1}=\bx_k+\epsilon_{k} \frac{N}{n}\nabla_{\bx}\widetilde  U_{\Psi_{\btheta_k}}(\bx_k)+\sqrt{2\tau \epsilon_{k}}\bw_{k}$, 
\item[(2)] Optimize $\bm{\theta}_{k+1}=\bm{\theta}_{k}+\omega_{k+1} \mathbb{ \widetilde H}(\bm{\theta}_{k}, \bx_{k+1}),$
\end{itemize}
where $\nabla_{\bx}\widetilde U_{\Psi_{\btheta}}(\cdot)$ is a stochastic gradient function of $\varpi_{\Psi_{\btheta}}(\cdot)$ to be detailed in Algorithm \ref{alg:ICSGLD}. $\mathbb{ \widetilde H}(\btheta,\bx):=\left(\mathbb{ \widetilde H}_1(\btheta,\bx), \cdots, \mathbb{ \widetilde H}_m(\btheta,\bx)\right)$ is random-field function where each entry follows
\begin{equation}
\label{ori_randomF}
    \mathbb{ \widetilde H}_i(\btheta,\bx)={\theta}^{\zeta}( J_{\widetilde U} (\bx))\left(1_{i= J_{\widetilde U}(\bx)}-{\theta}(i)\right), \text{where } J_{\widetilde U}(\bx)=\sum_{i=1}^m i 1_{u_{i-1}<\frac{N}{n} \widetilde U(\bx)\leq u_i}.
\end{equation}
Theoretically, CSGLD converges to a sampling-optimization equilibrium in the sense that $\btheta_{k}$ approaches to a fixed point $\btheta_{\infty}$ and the samples are drawn from the flattened density $\varpi_{\Psi_{\btheta_{\infty}}}(\bx)$. Notably, the mean-field system is \emph{globally stable} with a unique stable equilibrium point in a small neighborhood of $\btheta_{\infty}$. Moreover, such an appealing property holds even when $U(\bx)$ is non-convex.

 \begin{algorithm*}[tb]
   \caption{Interacting contour stochastic gradient Langevin dynamics algorithm (ICSGLD). $\{\MX_i\}_{i=1}^m$ is pre-defined partition and $\zeta$ is a hyperparameter. The update rule in distributed-memory settings and discussions of hyperparameters is detailed in section \ref{hyper_setup} in the supplementary material. 
   }
   \label{alg:ICSGLD}
\begin{algorithmic}
   \STATE{\bfseries [1.] (Data subsampling)} Draw a mini-batch data $\mathcal{B}_k$ from $\mathcal{D}$, and compute stochastic gradients $\nabla_{\bx}\widetilde U(\bx_k^{(p)})$ and energies $\widetilde U(\bx_k^{(p)})$ for each $\bx^{(p)}$, where $p\in\{1,2,\cdots, P\}$, $|\mathcal{B}_k|=n$, and $|\mathcal{D}|=N$.

   \STATE {\bfseries [2.] (Parallel simulation)}
  Sample $\bx_{k+1}^{\pop P}:=(\bx_{k+1}^{(1)}, \bx_{k+1}^{(2)}, \cdots, \bx_{k+1}^{(P)})^{\top}$ based on SGLD and $\btheta_k$
   \begin{equation}
   \bx_{k+1}^{\pop P}=\bx_k^{\pop P}+\epsilon_{k} \frac{N}{n}\nabla_{\bx}\widetilde  \bU_{\Psi_{\btheta_k}}(\bx_k^{\pop P})+\sqrt{2\tau \epsilon_{k}}\bw_{k}^{\pop P},
   \end{equation}
   where $\epsilon_{k}$ is the learning rate, $\tau$ is the temperature, $\bw_{k}^{\pop P}$ denotes $P$ independent standard Gaussian vectors, $\nabla_{\bx}\widetilde \bU_{\Psi_{\btheta}}(\bx^{\pop P})=(\nabla_{\bx}\widetilde U_{\Psi_{\btheta}}(\bx^{(1)}), \nabla_{\bx}\widetilde U_{\Psi_{\btheta}}(\bx^{(2)}), \cdots, \nabla_{\bx}\widetilde U_{\Psi_{\btheta}}(\bx^{(P)}))^{\top}$, and $\nabla_{\bx}\widetilde U_{\Psi_{\btheta}}(\bx)= \left[1+ 
   \frac{\zeta\tau}{\Delta u}  \left(\log \theta({J}_{\widetilde U}(\bx))-\log\theta((J_{\widetilde U}(\bx)-1)\vee 1) \right) \right]  
    \nabla_{\bx} \widetilde U(\bx)$ for any $\bx\in\MX$.
  \STATE {\bfseries [3.] (Stochastic approximation)} Update the self-adapting parameter $\theta(i)$ for $i\in\{1,2,\cdots, m\}$ 
  \begin{equation}
  \begin{split}
  \label{updateeq}
 \theta_{k+1}(i)&={\theta}_{k}(i)+\omega_{k+1}\frac{1}{P}\sum_{p=1}^P {\theta}_{k}( J_{\widetilde U}(\bx_{k+1}^{(p)}))\left(1_{i=J_{\widetilde U}(\bx_{k+1}^{(p)})}-{\theta}_{k}(i)\right),
 \end{split}
 \end{equation}
 where $1_{A}$ is an indicator function that takes value 1 if the event $A$ appears and equals 0 otherwise. 

\vspace{-0.04in}
\end{algorithmic}
\end{algorithm*}

\section{Interacting contour stochastic gradient Langevin dynamics}
\vskip -0.05in
The major goal of interacting CSGLD (ICSGLD) is to improve the efficiency of CSGLD. In particular, the self-adapting parameter $\btheta$ is crucial for ensuring the sampler to escape from the local traps and traverse the whole energy landscape, and how to reduce the variability of $\btheta_k$'s is the key to the success of such a dynamic importance sampling algorithm. To this end, we propose an efficient variance reduction scheme via interacting parallel systems to improve the accuracy of  $\btheta_{k}$'s.

\subsection{Interactions in parallelism}

Now we first consider a na\"{i}ve parallel sampling scheme with $P$ chains as follows
\begin{equation*}
\small
    \bx_{k+1}^{\pop P}=\bx_k^{\pop P}+\epsilon_{k} \frac{N}{n}\nabla_{\bx}\widetilde  \bU_{\Psi_{\btheta_k}}(\bx_k^{\pop P})+\sqrt{2\tau \epsilon_{k}}\bw_{k}^{\pop P},
\end{equation*}
where $\bx^{\pop P}=(\bx^{(1)}, \bx^{(2)}, \cdots, \bx^{(P)})^{\top}$, $\bw_{k}^{\pop P}$ denotes $P$ independent standard Gaussian vectors, and $\widetilde \bU_{\Psi_{\btheta}}(\bx^{\pop P})=(\widetilde U_{\Psi_{\btheta}}(\bx^{(1)}), \widetilde U_{\Psi_{\btheta}}(\bx^{(2)}), \cdots, \widetilde U_{\Psi_{\btheta}}(\bx^{(P)}))^{\top}$.

Stochastic approximation aims to find the solution $\btheta$ of the mean-field system $h(\btheta)$ such that 
\begin{equation*}
\small
\begin{split}
    h(\btheta)&=\int_{\MX} \widetilde H(\bm{\theta}, \bm{\bx}) \varpi_{\btheta}(d\bx)=0,
\end{split}
\end{equation*}
where $\varpi_{\btheta}$ is the invariant measure simulated via SGLD that approximates $\varpi_{\Psi_{\btheta}}$ in (\ref{flat_density}) and $\widetilde H(\bm{\theta}, \bm{\bx})$ is the novel random-field function to be defined later in (\ref{new_randomF}). Since $h(\btheta)$ is observable only up to large random perturbations (in the form of $\widetilde H(\bm{\theta}, \bm{\bx})$), the optimization of $\btheta$ based on isolated random-field functions may not be efficient enough. However, due to the \emph{conditional independence} of $\bx^{(1)}, \bx^{(2)},\cdots, \bx^{(P)}$ in parallel sampling, it is very natural to consider a Monte Carlo average
\begin{equation}
\label{monte_carlo_avg}
\small
\begin{split}
    h(\btheta)&=\frac{1}{P}\sum_{p=1}^P\int_{\MX} \widetilde H(\bm{\theta}, \bm{\bx}^{(p)}) \varpi_{\btheta}(d\bx^{(p)})=0.
\end{split}
\end{equation}
Namely, we are considering the following stochastic approximation scheme
\begin{equation}
\label{SA_step}
    \bm{\theta}_{k+1}=\bm{\theta}_{k}+\omega_{k+1} \widetilde \bH(\bm{\theta}_{k}, \bx_{k+1}^{\pop P}),
\end{equation}
where $\widetilde \bH(\bm{\theta}_{k}, \bx_{k+1}^{\pop P})$ is an interacting random-field function $ \widetilde \bH(\bm{\theta}_{k}, \bx_{k+1}^{\pop P})=\frac{1}{P}\sum_{p=1}^P \widetilde H(\bm{\theta}_{k}, \bx_{k+1}^{(p)})$.
Note that the Monte Carlo average is very effective to reduce the variance of the interacting random-field function $\widetilde \bH(\bm{\theta}, \bm{\bx}^{\pop P})$ based on the conditionally independent random field functions. Moreover, each chain shares with others only a very short message during each iteration. Therefore, the interacting parallel system is well suited for distributed computing, where the \emph{implementations and communication costs} are further detailed in section \ref{communication_cost} in the supplementary material. By contrast, each chain of the non-interacting parallel CSGLD algorithm deals with the parameter $\btheta$ and a large-variance random-field function $\widetilde H(\bm{\theta}, \bx)$ individually, leading to coarse estimates in the end.

Formally, for the population/ensemble interaction scheme (\ref{SA_step}),  we define a novel random-field function $\widetilde H(\btheta,\bx)=(\widetilde H_1(\btheta,\bx), \widetilde H_2(\btheta,\bx), \cdots, \widetilde H_m(\btheta,\bx))$, where each component satisfies
\begin{equation}
\label{new_randomF}
\widetilde H_i(\btheta,\bx)={\theta}( J_{\widetilde U} (\bx))\left(1_{i= J_{\widetilde U}(\bx)}-{\theta}(i)\right).  
\end{equation}
As shown in Lemma \ref{convex_main}, the corresponding mean-field function proposes to converge to a different fixed point $\btheta_{\star}$, s.t.
\begin{equation}
\small
\label{relation_two_thetas}
    \theta_{\star}(i)\propto\left(\int_{\MX_i} e^{-\frac{U(\bx)}{\tau}}d\bx\right)^{\frac{1}{\zeta}}\propto \btheta_{\infty}^{\frac{1}{\zeta}}(i).
\end{equation}
A large data set often renders the task of estimating $\btheta_{\infty}$ numerically challenging. By contrast, we resort to a different solution by estimating $\btheta_{\star}$ instead based on a large value of $\zeta$. The proposed algorithm is summarized in Algorithm \ref{alg:ICSGLD}. For more study on the scalablity of the new scheme, we leave the discussion in section \ref{scalability}.

\subsection{Related works}

Replica exchange SGLD \citep{deng2020, deng_VR} has successfully extended the traditional replica exchange \citep{PhysRevLett86, Geyer91, parallel_tempering05} to big data problems. However, it works with two chains only and has a low swapping rate. As shown in Figure \ref{replica_vs_contour}(a), a na\"{i}ve extension of multi-chain replica exchange SGLD yields low communication efficiency. Despite some recipe in the literature \citep{Katzgraber06, Elmar08, Doucet19}, how to conduct multi-chain replica exchange with low-frequency swaps is still an open question.

\begin{figure*}[!ht]
  \centering
  \vskip -0.17in
  \subfigure[Replica Exchange (parallel tempering)]{\includegraphics[scale=0.44
  ]{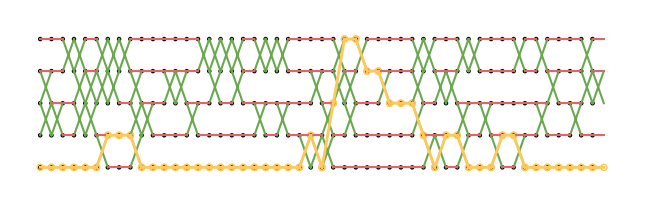}}\label{fig: 3a}\quad\quad
  \hspace{-0.5cm}
  \subfigure[Interacting contour SGLD (ICSGLD)]{\includegraphics[scale=0.44]{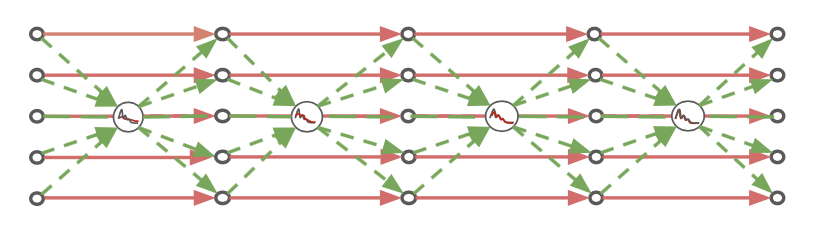}}\label{fig: 3b}
  \vspace{-0.5em}
  \caption{A comparison of communication costs between replica exchange (RE) and ICSGLD. We see RE takes many iterations to swap with all the other chains; by contrast, ICSGLD possesses a pleasingly parallel mechanism where the only cost comes from sharing a light message.}
  \label{replica_vs_contour}
  \vspace{-0.05in}
\end{figure*}

Stein variational gradient descent (SVGD) \citep{SVGD} is a popular approximate inference method to drive a set of particles for posterior approximation. In particular, repulsive forces are proposed to prevent particles to collapse together into neighboring regions, which resembles our strategy of penalizing over-visited partition. However, SVGD tends to underestimate the uncertainty given a limited number of particles. Moreover, the quadratic cost in kernel matrix computation further raises the scalability concerns as more particles are proposed.

Admittedly, ICSGLD is not the first interacting importance sampling algorithm. For example, 
a population stochastic approximation Monte Carlo (pop-SAMC) algorithm  has been proposed in \cite{SongWL2014}, and 
an interacting particle Markov chain Monte Carlo (IPMCMC) algorithm has been proposed  in \cite{IPMCMC}.
A key difference between our algorithm and others is that our algorithm is mainly devised for big data problems. The IPMCMC and pop-SAMC are gradient-free samplers, which are hard to be adapted to high-dimensional big data problems.

Other parallel SGLD methods \citep{Ahn14_icml, chen16_distributed} aim to reduce the computational cost of gradient estimations in distributed computing, which, however, does not consider interactions for accelerating the convergence. \cite{Li19_v2} proposed asynchronous protocols to reduce communication costs when the master aggregates model parameters from all workers. Instead, we don't communicate the parameter $\bx\in\mathbb{R}^d$ but only share $\btheta\in \mathbb{R}^m$ and the indices, where $m\ll d$.

Our work also highly resembles the well-known Federated Averaging (FedAvg) algorithm \citep{lhy+20, FA-LD}, except that the stochastic gradient $\widetilde U(\bx)$ is replaced with the random field function $\widetilde H(\btheta, \bx)$ and we only share the low-dimensional latent vector $\btheta$. Since privacy concerns and communication cost are not major bottlenecks of our problem, we leave the study of taking the Monte Carlo average in Eq.(\ref{monte_carlo_avg}) every $K>1$ iterations for future works.

\section{Convergence properties}
\label{all_proof}
\vskip -0.05in
To study theoretical properties of ICSGLD, we first show a local stability property that is well-suited to big data problems, and then we present the asymptotic normality for the stochastic approximation process in \emph{mini-batch settings}, which eventually yields the desired result that ICSGLD is asymptotically more efficient than a single-chain CSGLD with an equivalent computational cost.

\subsection{Local stability for non-linear mean-field systems in big data}

The first obstacle  for the theoretical study is to approximate the components of $\btheta_{\infty}$  corresponding to the high energy region.  
 To get around this issue, the random field function $\widetilde H(\btheta,\bx)$ in (\ref{new_randomF}) is adopted to estimate a different target $\btheta_{\star}\propto \btheta_{\infty}^{\frac{1}{\zeta}}$. As detailed in Lemma \ref{convex_appendix} in the supplementary material, the mean-field equation is now formulated as follows
\begin{equation}
\label{h_i_theta_main}
     h_i(\btheta)\propto \theta_{\star}^{\zeta}(i)-{\left(\theta(i) C_{\btheta}\right)^{\zeta}}+{\text{perturbations}},
\end{equation}
where $C_{\btheta}=\bigg(\frac{\widetilde Z_{\zeta,\btheta}}{\widetilde Z_{\zeta, \btheta_{\star}}^{\zeta}}\bigg)^{\frac{1}{\zeta}}$ and $\widetilde Z_{\zeta,\btheta}=\sum_{k=1}^m  \frac{\int_{\MX_k} \pi(\bx)d\bx}{\theta^{\zeta-1}(k)}$. We see that (\ref{h_i_theta_main}) may not be linearly stable as in \cite{CSGLD}. Although the solution of the mean-field system $h(\btheta)=0$ is still unique, there may exist unstable invariant subspaces, leading us to consider the local properties. For a proper initialization of $\btheta$, which can be achieved by pre-training the model long enough time through SGLD, the mean value theorem implies a linear property in a local region
\begin{equation*}
     h_i(\btheta)\propto \theta_{\star}(i)-\theta(i)+{\text{perturbations}}.
\end{equation*}
Combining the perturbation theory \citep{Eric}, we present the following stability result:
\begin{lemma}[Local stability, informal version of Lemma \ref{convex_appendix}] \label{convex_main} 
Assume Assumptions A1-A4 (given in the supplementary material) hold. For any properly initialized $\btheta$, we have $\langle h(\btheta), \btheta - \widehat\btheta_{\star}\rangle \leq  -\phi\|\btheta - \widehat\btheta_{\star}\|^2$,  where $\widehat \btheta_{\star}=\btheta_{\star}+\mathcal{O}\left(\sup_{\bx}\Var(\xi_n(\bx))+\epsilon+\frac{1}{m}\right)$, $\btheta_{\star}\propto \btheta_{\infty}^{\frac{1}{\zeta}}$ ,
$\phi>0$, $\epsilon$ denotes a learning rate, and $\xi_n(\bx)$ denotes the noise in the stochastic energy estimator of batch size $n$ and $\Var(\cdot)$ denotes the variance. 
\end{lemma}

By justifying the drift conditions of the adaptive transition kernel and relevant smoothness properties, we can prove the existence and regularity properties of the solution of the Poisson's equation in Lemma \ref{lyapunov_ori} in the supplementary material. In what follows, we can control the fluctuations in stochastic approximation and eventually yields the $L^2$ convergence. 
\begin{lemma}[$L^2$ convergence rate, informal version of Lemma \ref{latent_convergence_appendix}]
\label{latent_convergence_main}
Given standard Assumptions A1-A5. $\btheta_k$ converges to $\widehat\btheta_{\star}$,
where $\widehat\btheta_{\star}=\btheta_{\star}+\mathcal{O}\left(\sup_{\bx}\Var(\xi_n(\bx))+\epsilon+\frac{1}{m}\right)$, such that
\begin{equation*}
    \E\left[\|\bm{\theta}_{k}-\widehat\btheta_{\star}\|^2\right]= \mathcal{O}\left(\omega_{k}\right).
\end{equation*}
\end{lemma}

The result differs from Theorem 1 of \cite{CSGLD} in that the biased fixed point $\widehat\btheta_{\star}$ instead of $\btheta_{\star}$ is treated as the equilibrium of the continuous system, which provides us a user-friendly proof. Similar techniques have been adopted by \citet{Alain17, Xu18}. Although the global stability \citep{CSGLD} may be sacrificed when $\zeta\neq 1$ based on Eq.(\ref{new_randomF}), $\btheta_{\star}\propto \btheta_{\infty}^{\frac{1}{\zeta}}$ is much easier to estimate numerically for any $i$ that yields $0<\btheta_{\infty}(i)\ll 1$ based on a large $\zeta>1$.

\subsection{Asymptotic normality}

 To study the asymptotic behavior of $\omega_k^{-\frac{1}{2}}(\btheta_k-\widehat\btheta_{\star})$, where $\widehat\btheta_{\star}$ is the equilibrium point s.t. $\widehat\btheta_{\star}=\btheta_{\star}+\mathcal{O}\left(\Var(\xi_n(\bx))+\epsilon+\frac{1}{m}\right)$, we consider  a fixed step size $\omega$ in the SA step for ease of explanation. Let $\bar\btheta_t$ denote the solution of the mean-field system in continuous time ($\bar\btheta_0=\btheta_0$), and rewrite the single-chain SA step (\ref{SA_step}) as follows
\begin{equation*}
\begin{split}
\label{ga_main}
    \btheta_{k+1}-\bar\btheta_{(k+1)\omega}&=\btheta_k-\bar\btheta_{k\omega}+\omega\left(H(\btheta_{k}, \bx_{k+1})-H(\bar\btheta_{k\omega}, \bx_{k+1})\right)\\
    &\quad+\omega\left(H(\bar\btheta_{k\omega}, \bx_{k+1})-h(\bar\btheta_{k\omega})\right)-\left(\bar\btheta_{(k+1)\omega}-\bar\btheta_{k\omega}-\omega  h(\bar\btheta_{k\omega})\right).
\end{split}
\end{equation*}

Further, we set $\widetilde\btheta_{k\omega}:= \omega^{-\frac{1}{2}}(\btheta_{k}-\bar\btheta_{k\omega})$. Then the stochastic approximation differs from the mean field system in that
\begin{equation*}
\footnotesize
\begin{split}
    \widetilde\btheta_{(k+1)\omega}&= \underbrace{\omega^{\frac{1}{2}}\sum_{i=0}^k \left(H(\btheta_{i}, \bx_{i+1})-H(\bar\btheta_{i\omega}, \bx_{i+1})\right)}_{\text{I: perturbations}}+\omega^{\frac{1}{2}}\sum_{i=0}^k \underbrace{\left(H(\bar\btheta_{i\omega}, \bx_{i+1})-h(\bar\btheta_{i\omega})\right)}_{\text{II: martingale} \ \mathcal{M}_i}-\omega^{\frac{1}{2}}\cdot\text{remainder}\\
    &\approx \omega^{\frac{1}{2}}\sum_{i=0}^k h_{\btheta}(\btheta_{i\omega})  \underbrace{(\btheta_i-\bar\btheta_{i\omega})}_{\approx \omega^{\frac{1}{2}} \widetilde \btheta_{i\omega}}+\omega^{\frac{1}{2}}\sum_{i=0}^k \mathcal{M}_i\approx \int_{0}^{(k+1)\omega}h_{\btheta}(\bar\btheta_{s})\widetilde\btheta_{s}ds+\int_0^{(k+1)\omega} \bR^{\frac{1}{2}}(\bar\btheta_s)d\bW_s,
\end{split}
\end{equation*}
where $h_{\btheta}(\btheta):=\frac{d}{d\btheta} h(\btheta)$ is a matrix, $\bW\in\mathbb{R}^m$ is a standard Brownian motion, the last term follows from a certain central limit theorem \citep{Albert90} and $\bR$ denotes the covariance matrix of the random-field function s.t. $\bR(\btheta):=\sum_{k=-\infty}^{\infty} \cov_{\btheta}(H(\btheta, \bx_k), H(\btheta, \bx_0))$. 

We expect the weak convergence of $\bU_k$ to the stationary distribution of a diffusion
\begin{equation}
\label{slde_main}
    d\bU_t=h_{\btheta}(\btheta_t) \bU_t dt + \bR^{1/2}(\btheta_t)d\bW_t,
\end{equation}
where $\bU_t=\omega_t^{-1/2}(\btheta_t-\widehat\btheta_{\star})$.  Given that $\btheta_t$ converges to $\widehat\btheta_{\star}$ sufficiently fast and the local linearity of $h_{\btheta}$, the diffusion (\ref{slde_main}) resembles the Ornstein–Uhlenbeck process and yields the following solution
\begin{equation*}
\small
    \bU_t\approx e^{-th_{\btheta}(\widehat\btheta_{\star})}\bU_0+\int_0^t e^{-(t-s)h_{\btheta}(\widehat\btheta_{\star})}\circ \bR(\widehat\btheta_{\star}) d\bW_s.
\end{equation*}
Then we have the following theorem, whose  formal proof is given in section \ref{proof_theorem_1}.
\begin{theorem}[Asymptotic Normality]
\label{Asymptotic}
Assume Assumptions A1-A5 (given in the supplementary material) hold. We have the following weak convergence
\begin{equation*}
\begin{split}
    \omega_k^{-1/2}(\btheta_k-\widehat\btheta_{\star})\Rightarrow\mathcal{N}(0, \bSigma), \text{ where  } \bSigma=\int_0^{\infty} e^{t h_{\btheta_{\star}}}\circ \bR\circ  e^{th^{\top}_{\btheta_{\star}}}dt, h_{\btheta_{\star}}=h_{\btheta}(\widehat\btheta_{\star}).
\end{split}
\end{equation*}
\end{theorem}

\subsection{Interacting parallel chains are more efficient}

For clarity, we first denote an estimate of $\btheta$ based on ICSGLD with $P$ interacting parallel chains by $\btheta_k^{P}$ and denote the estimate based on a single-long-chain CSGLD by $\btheta_{kP}$.

Note that Theorem \ref{Asymptotic} holds for any step size $\omega_k=\mathcal{O}(k^{-\alpha})$, where $\alpha\in (0.5, 1]$. If we simply run a single-chain CSGLD algorithm with $P$ times of iterations, by  Theorem \ref{Asymptotic},  
\begin{equation*}
\begin{split}
    \omega_{kP}^{-1/2}(\btheta_{kP}-\widehat\btheta_{\star})\Rightarrow\mathcal{N}(0, \bSigma).
\end{split}
\end{equation*}
As to ICSGLD, since the covariance $\bSigma$ relies on $\bR$, which depends on the covariance of the martingale $\{\mathcal{M}_i\}_{i\geq 1}$, the conditional independence of $\bx^{(1)},\bx^{(2)},\cdots, \bx^{(P)}$ naturally results in an efficient variance reduction such that 

\begin{corollary}[Asymptotic Normality for ICSGLD]
 Assume the same assumptions. For ICSGLD with $P$ interacting chains, we have the following weak convergence
\begin{equation*}
\begin{split}
    \omega_k^{-1/2}(\btheta_k^P-\widehat\btheta_{\star})\Rightarrow\mathcal{N}(0, \bSigma/P).
\end{split}
\end{equation*}
 \end{corollary}
That is, under a similar computational budget, we have $\frac{\|\Var(\btheta_{kP}-\widehat\btheta_{\star})\|_{\text{F}}}{\|\Var(\btheta_k^P-\widehat{\btheta}_*)\|_{\text{F}}}= \frac{w_{kP}}{w_k/P}\approx P^{1-\alpha}$.

 \begin{corollary}[Efficiency]
Given a decreasing step size $\omega_k=\mathcal{O}(k^{-\alpha})$, where $0.5 < \alpha<1$, \emph{ICSGLD is asymptotically more efficient than the single-chain CSGLD with an equivalent training cost.} 
\end{corollary}

In practice, slowly decreasing step sizes are often preferred in stochastic algorithms for a better non-asymptotic performance \citep{Albert90}.

\section{Experiments}
\vskip -0.05in

\subsection{Landscape exploration on MNIST via the scalable random-field function}

This section shows how the novel random-field function (\ref{new_randomF}) facilitates the exploration of multiple modes on the MNIST dataset\footnote[4]{The random-field function \citep{CSGLD} requires an extra perturbation term as discussed in section D4 in the supplementary material \citep{CSGLD}; therefore it is not practically appealing in big data.}, while the standard methods, such as stochastic gradient descent (SGD) and SGLD, only \emph{get stuck in few local modes}. To simplify the experiments, we choose a large batch size of 2500 and only pick the first five classes, namely digits from 0 to 4. The \emph{learning rate is fixed} to 1e-6 and the temperature is set to $0.1$ \footnote[2]{Data augmentation implicitly leads to a more concentrated posterior \citep{Florian2020, Aitchison2021}.}. We see from Figure \ref{Uncertainty_estimation_mnist}{\textcolor{red}{(a)}} that both SGD and SGLD lead to fast decreasing losses. By contrast, ICSGLD yields fluctuating losses that traverse freely between high energy and low energy regions. As the particles stick in local regions, the penalty of re-visiting these zones keeps increasing until \emph{a negative learning rate is injected} to encourage explorations.

\begin{figure}[htbp]
\vspace{-0.05in}
\small
 \begin{tabular}{cccc}
(a) Training Loss & (b)  SGD & (c) SGLD & (d) ICSGLD \\ 
\includegraphics[height=1.2in,width=1.2in]{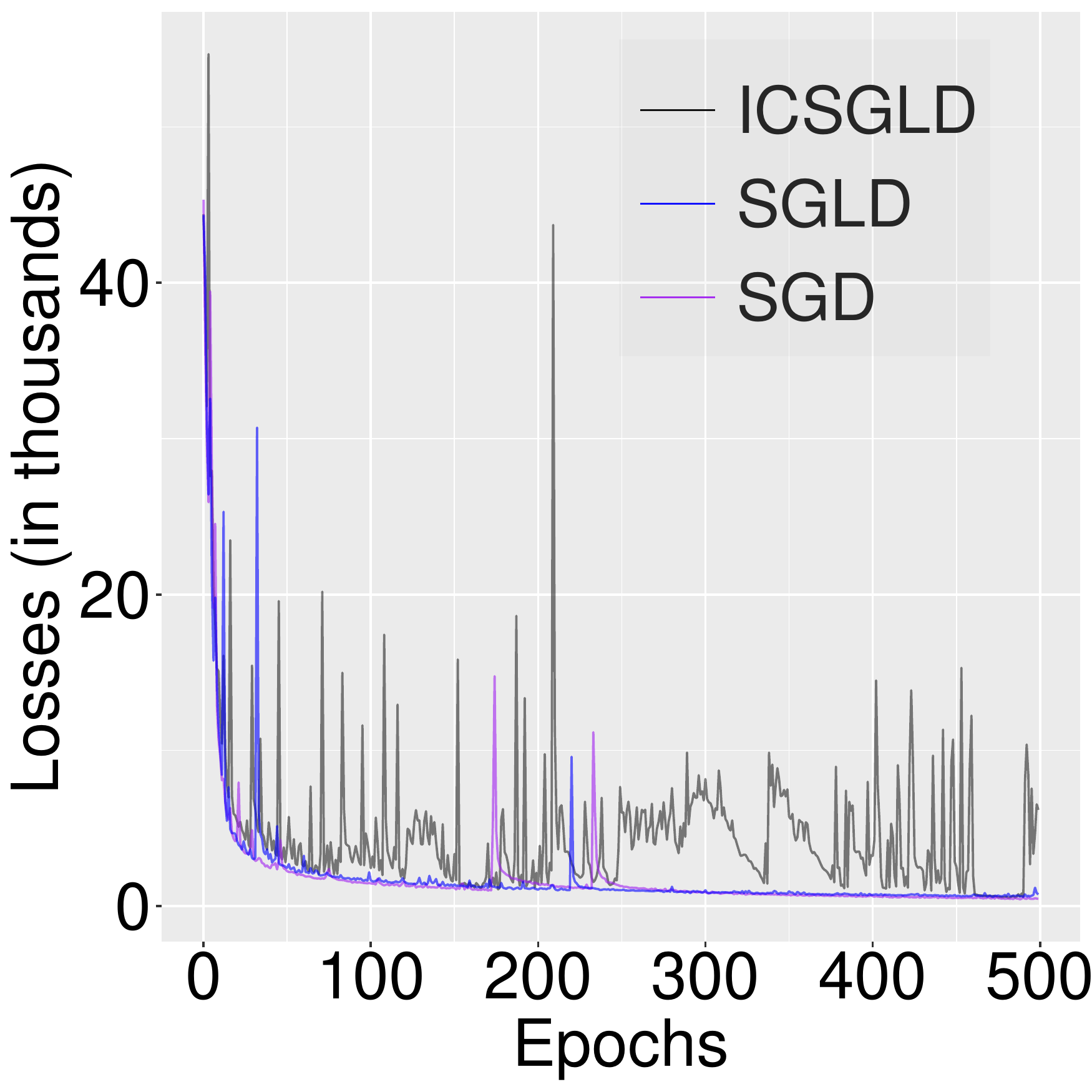} &
\includegraphics[height=1.25in,width=1.25in]{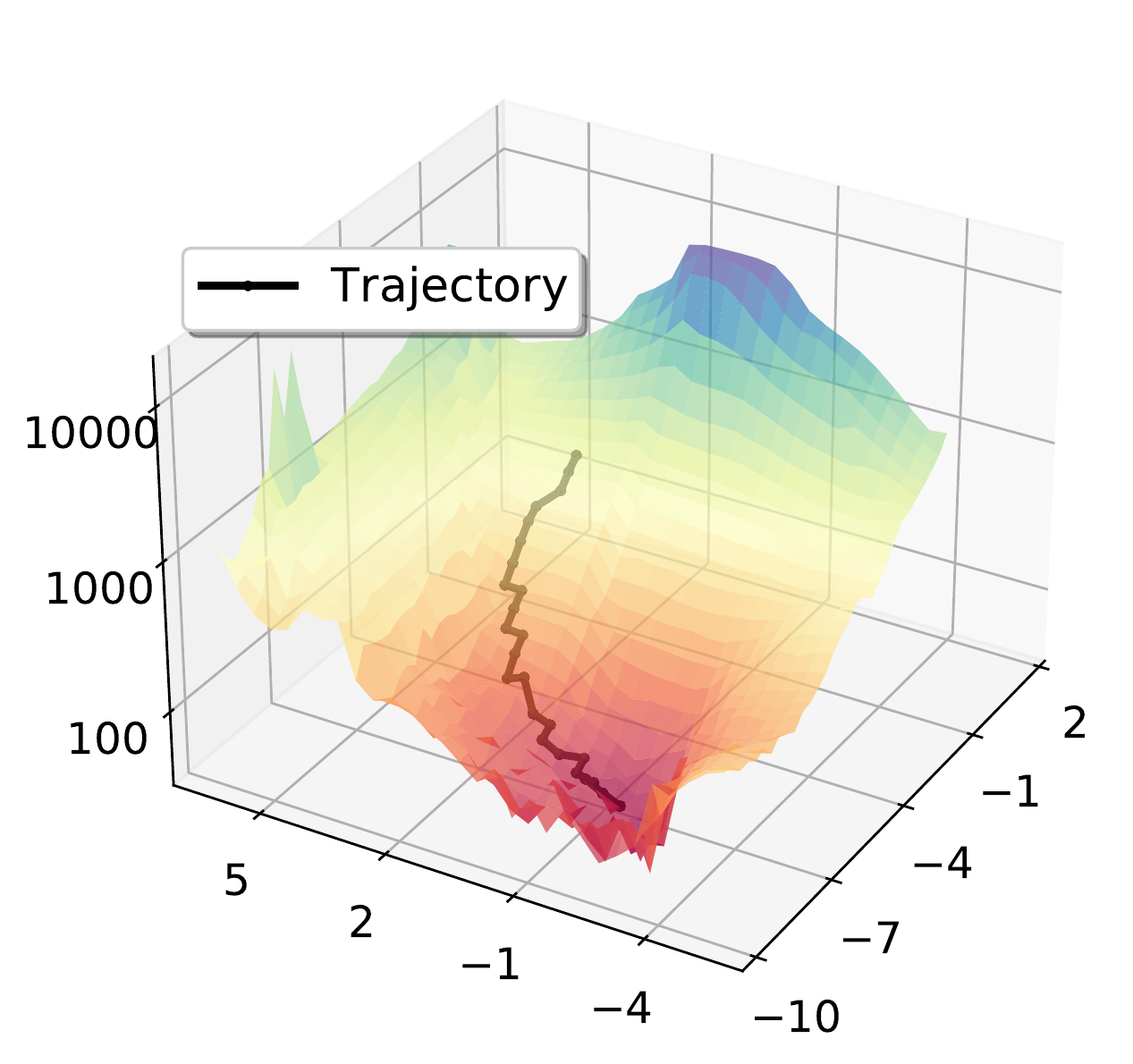} &
\includegraphics[height=1.25in,width=1.25in]{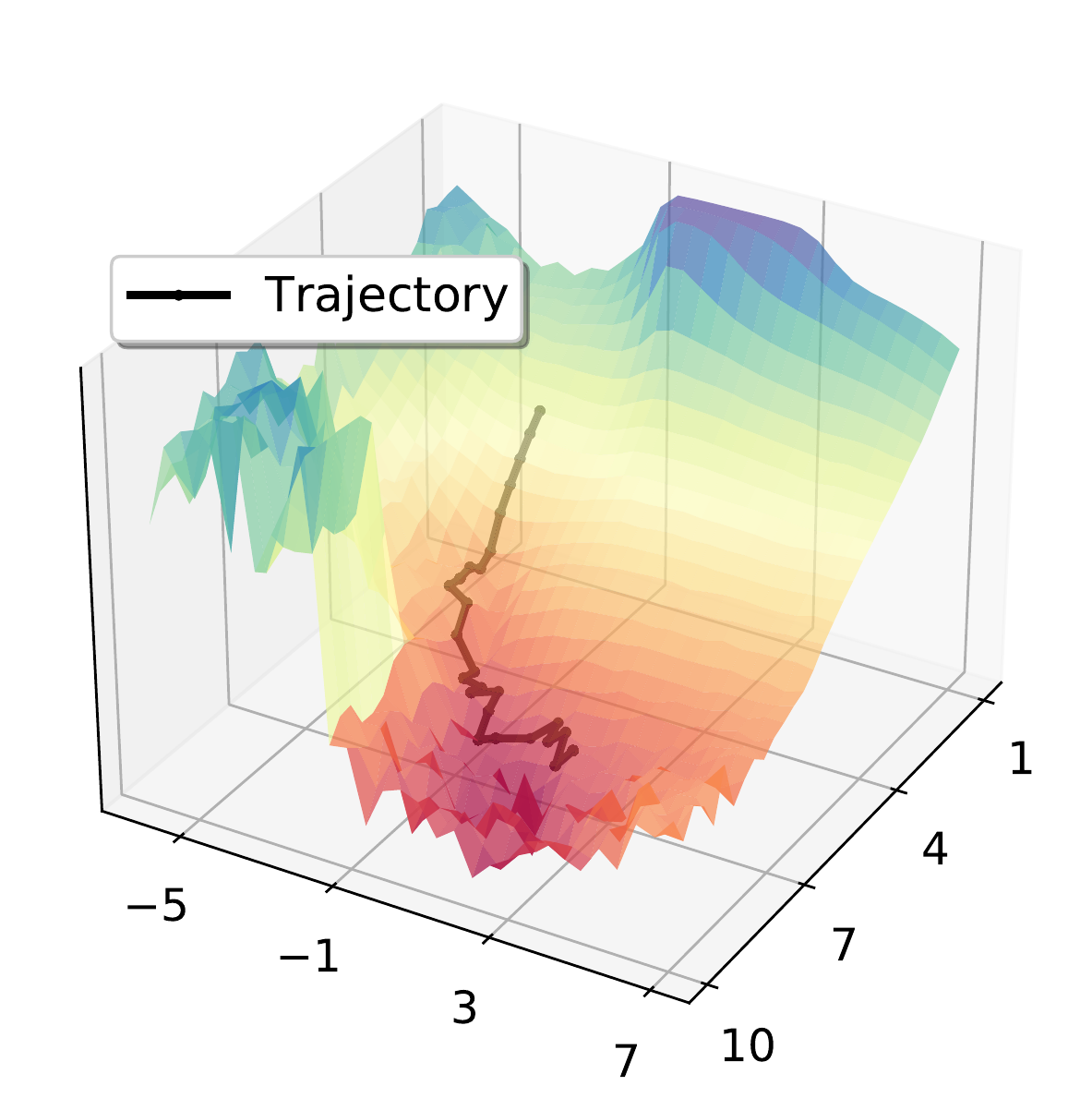} &
\includegraphics[height=1.25in,width=1.25in]{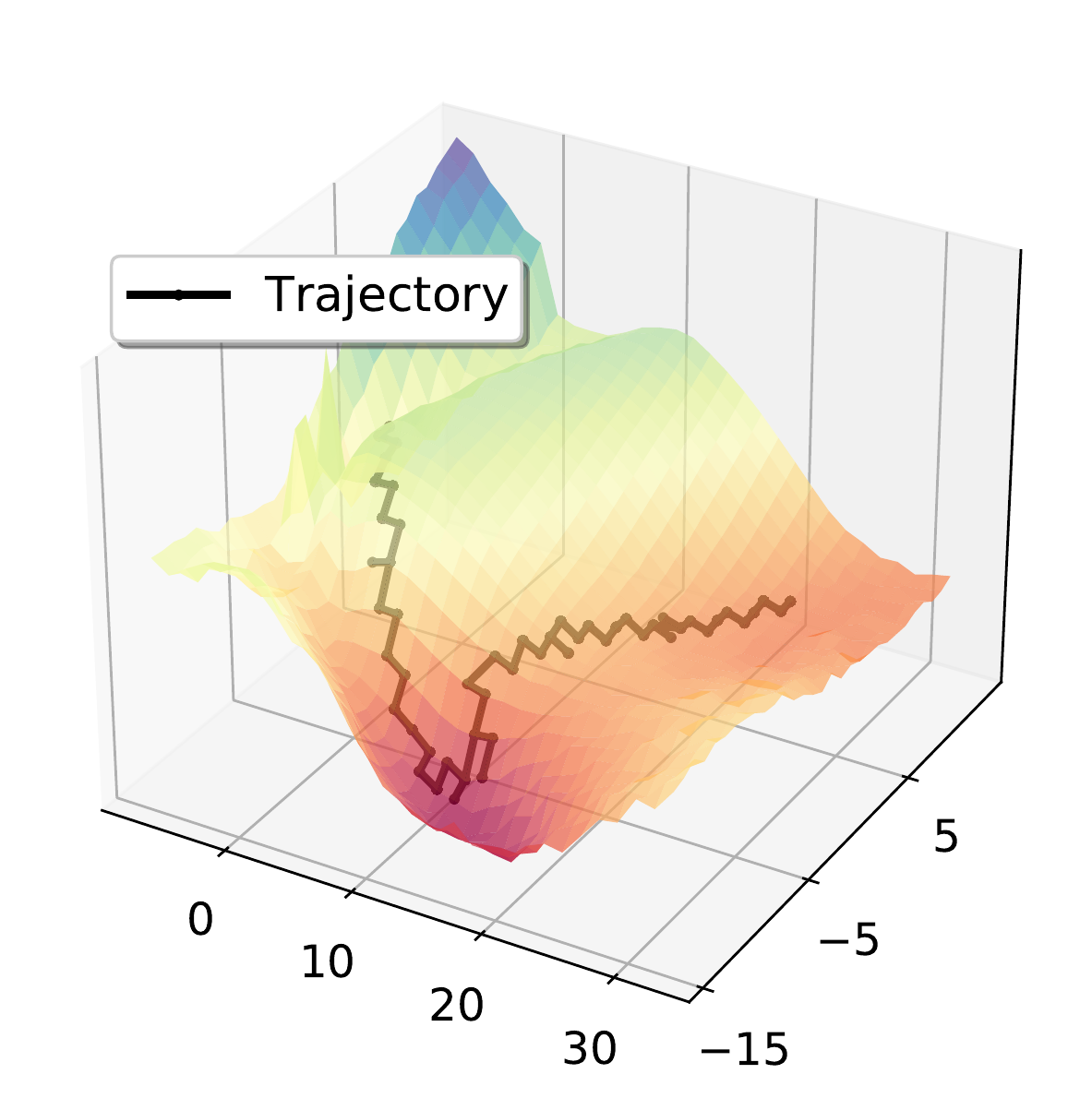}
\end{tabular}
\vspace{-0.1in}
  \caption{Visualization of mode exploration on a MNIST example based on different algorithms.} 
\label{Uncertainty_estimation_mnist}
\vspace{-0.05in}
\end{figure}

We conducted a singular value decomposition (SVD) based on the first two coordinates to visualize the trajectories: We first choose a domain that includes all the coordinates, then we recover the parameter based on the grid point and truncated values in other dimensions, and finally we fine-tune the parameters and present the approximate losses of the trajectories in Figure \ref{Uncertainty_estimation_mnist}{\textcolor{red}{(b-d)}}. We see SGD trajectories get stuck in a local region; SGLD \emph{exploits a larger region} but is still quite limited in the exploration; ICSGLD, instead, first converges to a local region and then \emph{escapes it once it over-visits this region}. This shows the strength of ICSGLD in the simulations of complex multi-modal distributions. More experimental details are presented in section \ref{mnist_appendix} of the supplementary material.

\subsection{Simulations of multi-modal distributions}

This section shows the acceleration effect of ICSGLD via a group of simulation experiments for a multi-modal distribution. The baselines 
include popular Monte Carlo methods such as 
 CSGLD, SGLD, cyclical SGLD (cycSGLD), replica exchange SGLD (reSGLD), and the particle-based SVGD.

The target multi-modal density is presented in Figure \ref{subfig:true}. Figure \ref{figure:simulation}{\textcolor{red}{(b-g)}} displays the empirical performance of all the testing methods: the vanilla SGLD with 5 parallel chains ($\times$P5) undoubtedly performs the worst in this example and fails to quantify the weights of each mode correctly; the single-chain cycSGLD with 5 times of iterations ($\times$T5) improves the performance but is still not accurate enough; reSGLD ($\times$P5) and SVGD ($\times$P5) have good performances, while the latter is quite costly in computations; ICSGLD ($\times$P5) does not only traverse freely over the rugged energy landscape, but also yields the most accurate approximation to the ground truth distribution. By contrast, CSGLD ($\times$T5) performs worse than ICSGLD and overestimates the weights on the left side. For the detailed setups, the study of convergence speed, and runtime analysis, we refer interested readers to section \ref{simulation_appendix} in the supplementary material. 
  \begin{figure*}[htbp]
  \vspace{-0.25in}
    \centering
    \subfigure[Truth]{
    \begin{minipage}[t]{0.13\linewidth}
    \centering
    \includegraphics[width=0.82in]{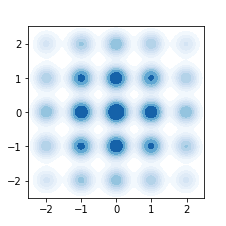}
    \label{subfig:true}
    \end{minipage}%
    }%
     \subfigure[SGLD]{
    \begin{minipage}[t]{0.13\linewidth}
    \centering
    \includegraphics[width=0.82in]{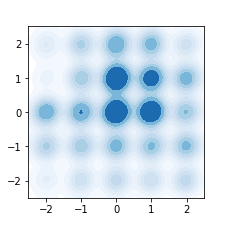}
    \label{subfig:pSGLD}
    \end{minipage}%
    }%
    \subfigure[cycSGLD]{
    \begin{minipage}[t]{0.13\linewidth}
    \centering
    \includegraphics[width=0.82in]{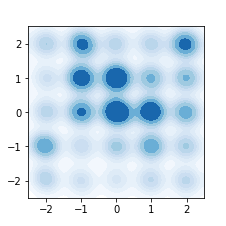}
    \label{subfig:cyclical SGLD}
    \end{minipage}%
    }%
    \subfigure[SVGD]{
    \begin{minipage}[t]{0.13\linewidth}
    \centering
    \includegraphics[width=0.83in]{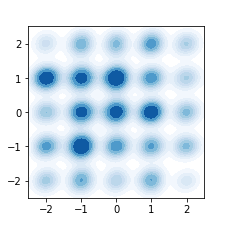}
    \label{subfig:pSVGD}
    \end{minipage}
    }%
    \subfigure[reSGLD]{
    \begin{minipage}[t]{0.13\linewidth}
    \centering
    \includegraphics[width=0.83in]{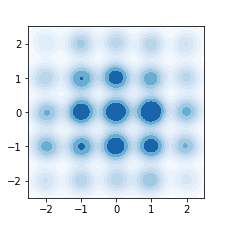}
    \label{subfig:reSGLD}
    \end{minipage}
    }%
    \subfigure[CSGLD]{
    \begin{minipage}[t]{0.13\linewidth}
    \centering
    \includegraphics[width=0.83in]{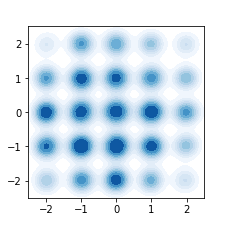}
    \label{subfig:CSGLD after}
    \end{minipage}
    }%
    \subfigure[ICSGLD]{
    \begin{minipage}[t]{0.13\linewidth}
    \centering
    \includegraphics[width=0.83in]{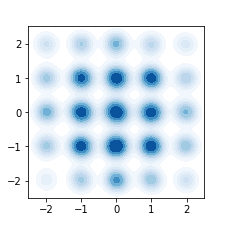}
    \label{subfig:iCSGLD after}
    \end{minipage}
    }%
  \vskip -0.15in
  \caption {Empirical behavior on a simulation dataset. Figure \ref{subfig:cyclical SGLD} and \ref{subfig:CSGLD after} show the simulation based on a single chain with 5 times of iterations ($\times$T5) and the others run 5 parallel chains ($\times$P5). }
  \label{figure:simulation}
  \vspace{-0.25in}
\end{figure*}

\subsection{Deep contextual bandits on mushroom tasks}

This section evaluates ICSGLD on the contextual bandit problem based on the UCI Mushroom data set as in \cite{bandits_showdown}. The mushrooms are assumed to arrive sequentially and the agent needs to take an action at each time step based on past feedbacks.
Our goal is to minimize the cumulative regret that measures the difference between the cumulative reward obtained by the proposed policy and optimal policy. We evaluate Thompson Sampling (TS) based on a variety of approximate inference methods for posterior sampling. We choose one $\epsilon$-greedy policy (EpsGreedy) based on the RMSProp optimizer with a decaying learning rate  \citep{bandits_showdown} as a baseline. Two variational methods, namely stochastic gradient descent with a constant learning rate (ConstSGD) \citep{Mandt} and Monte Carlo Dropout (Dropout) \citep{Gal16b} are compared to approximate the posterior distribution. For the sampling algorithms, we include preconditioned SGLD (pSGLD) \citep{Li16}, preconditioned CSGLD (pCSGLD) \citep{CSGLD}, and preconditioned ICSGLD (pICSGLD). Note that all the algorithms run 4 parallel chains with average outputs ($\times$P4) except that pCSGLD runs a single-chain with 4 times of computational budget ($\times$T4). 
For more details, we refer readers to section \ref{bandit_mushroom} in the supplementary material.

Figure \ref{mushroom} shows that EpsGreedy $\times$P4 tends to explore too much for a long horizon as expected; ConstSGD$\times$P4 and Dropout$\times$P4 perform poorly in the beginning but eventually outperform EpsGreedy $\times$P4 due to the inclusion of uncertainty for exploration, whereas the uncertainty seems  to be inadequate due to the nature of variational inference. By contrast, pSGLD$\times$P4 significantly \begin{wrapfigure}{r}{0.3\textwidth}
   \begin{center}
   \vskip -0.2in
     \includegraphics[width=0.3\textwidth]{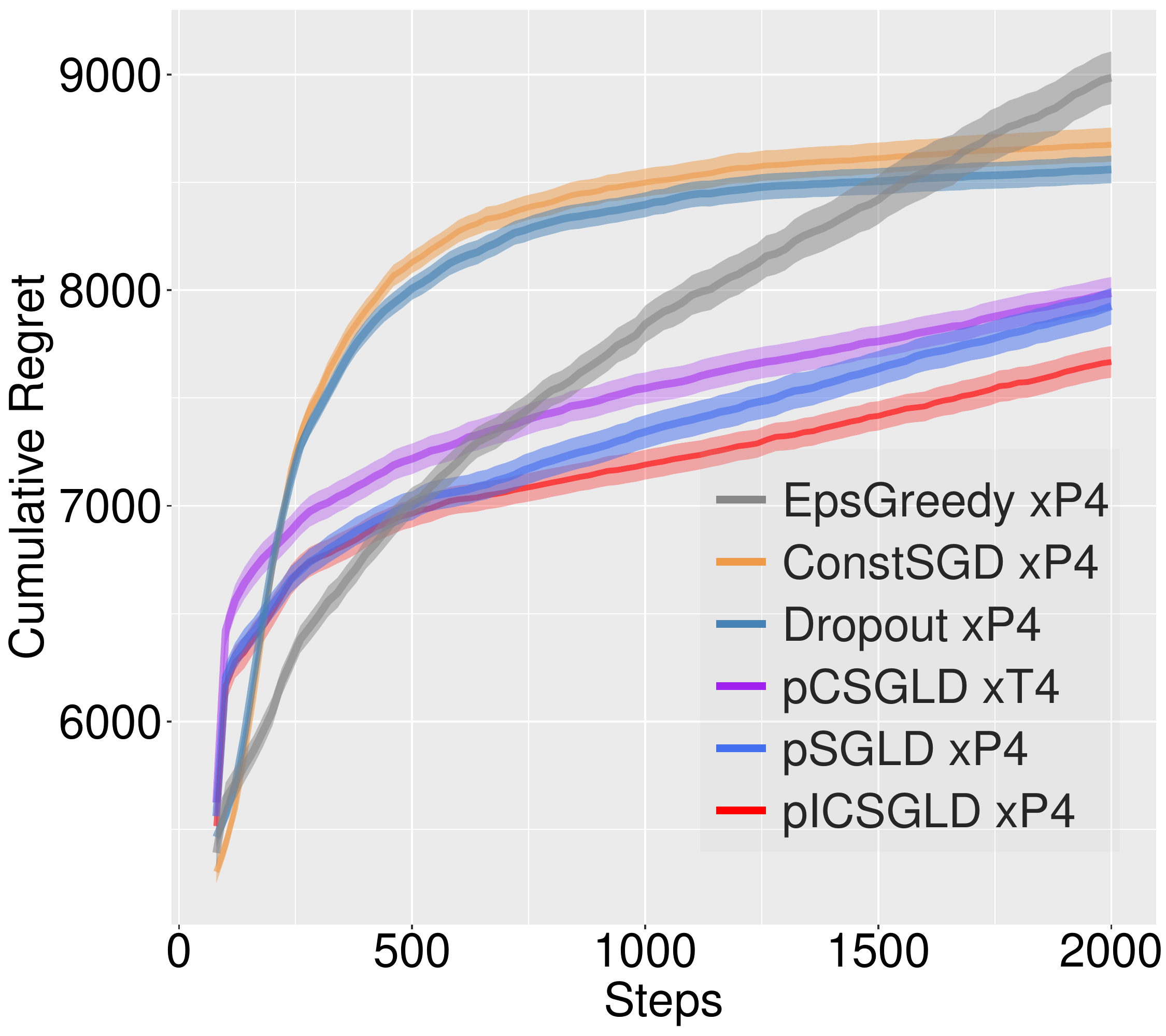}
   \end{center}
   \vskip -0.2in
   \caption{Cumulative regret on the mushroom task. }
   \label{mushroom}
\end{wrapfigure}  outperforms the variational methods by considering preconditioners within an exact sampling framework (SGLD).  As a unique algorithm that runs in a single-chain manner, pCSGLD$\times$T4 leads to the worst performance due to the inefficiency in learning the self-adapting parameters, fortunately, pCSGLD$\times$T4 slightly outperform pSGLD$\times$P4 in the later phase with the help of the well-estimated self-adapting parameters. Nevertheless, pICSGLD$\times$P4 propose to optimize the shared self-adapting parameters at the same time, which in turn greatly contributes to the simulation of the posterior. As a result, pICSGLD$\times$P4 consistently shows the lowest regret excluding the very early period. This shows the great superiority of the interaction mechanism in learning the self-adapting parameters for accelerating the simulations.

\subsection{Uncertainty estimation}

This section evaluates the qualify of our algorithm in uncertainty quantification. For model architectures, we use residual networks (ResNet) \citep{kaiming15} and a wide ResNet (WRN) \citep{wide_residual};  we choose 20, 32, and 56-layer ResNets (denoted by ResNet20, et al.) and a WRN-16-8 network, a 16-layer WRN that is 8 times wider than ResNet16. We train the models on CIFAR100, and report the test accuracy (ACC) and test negative log-likelihood (NLL) based on 5 trials with standard error. For the out-of-distribution prediction performance, we test the well-trained models in Brier scores (Brier) \footnote{The Brier score measures the mean squared error between the predictive and actual probabilities.} on the Street View House Numbers dataset (SVHN). 

Due to the wide adoption of momentum stochastic gradient descent (M-SGD), we use stochastic gradient Hamiltonian Monte Carlo (SGHMC) \citep{Chen14} as the baseline sampling algorithm and denote the interacting contour SGHMC by ICSHMC. In addition, we include several high performing baselines, such as SGHMC with cyclical learning rates (cycSGHMC) \citep{ruqi2020}, SWAG based on cyclic learning rates of 10 cycles (cycSWAG) \citep{swag} and variance-reduced replica exchange SGHMC (reSGHMC) \citep{deng_VR}. For a fair comparison, ICSGLD also conducts variance reduction on the energy function to alleviate the bias. Moreover, a large $\zeta=3\times 10^{6}$ is selected, which only induces mild gradient multipliers ranging from $-1$ to $2$ to penalize over-visited partitions. We don't include SVGD \citep{SVGD} and SPOS \citep{SPOS} for scalability reasons. A batch size of 256 is selected. We run 4 parallel processes ($\times$P4) with 500 epochs for M-SGD, reSGHMC and ICSGHMC and run cycSGHMC and cycSWAG 2000 epochs ($\times$T4) based on a single process with 10 cycles.  Refer to section \ref{UQ_appendix} of the supplementary material for the detailed settings.

\begin{table*}[ht]
\begin{sc}
\vspace{-0.15in}
\caption[Table caption text]{Uncertainty estimations on CIFAR100 and SVHN. } \label{UQ_test}
\vspace{-0.1in}
\small
\begin{center} 
\begin{tabular}{c|ccc|ccc}
\hline
\multirow{2}{*}{Model} & \multicolumn{3}{c|}{R\upshape{es}N\upshape{et}20} & \multicolumn{3}{c}{R\upshape{es}N\upshape{et}32}  \\
\cline{2-7}
 & ACC (\%) & NLL & B\upshape{rier} (\textperthousand) & ACC (\%) & NLL & B\upshape{rier} (\textperthousand) \\
\hline
\hline
\upshape{cyc}SGHMC$\times$T4  & 75.41$\pm$0.10 & 8437$\pm$30 & 2.91$\pm$0.13 & 77.93$\pm$0.17 & 7658$\pm$19 & 3.29$\pm$0.13 \\ 
\upshape{cyc}SWAG$\times$T4   & 75.46$\pm$0.11 & 8419$\pm$26 & 2.78$\pm$0.12  & 77.91$\pm$0.15 & 7656$\pm$22 & 3.19$\pm$0.14 \\ 
\hline
\hline
M-SGD$\times$P4 & 76.01$\pm$0.12  & 8175$\pm$25  &  2.58$\pm$0.08 &  78.41$\pm$0.12 & 7501$\pm$23 & 2.77$\pm$0.15  \\ 
\upshape{re}SGHMC$\times$P4 & 76.15$\pm$0.16  & 8196$\pm$27  & 2.73$\pm$0.10  & 78.57$\pm$0.07 & 7454$\pm$15 & 3.04$\pm$0.09   \\ 
ICSGHMC$\times$P4 & \textbf{76.34$\pm$0.15}  & \textbf{8076$\pm$31}  & \textbf{2.54$\pm$0.14} & \textbf{78.72$\pm$0.16}  & \textbf{7406$\pm$29}  & \textbf{2.76$\pm$0.15} \\ 
\hline
\hline
\multirow{2}{*}{Model} & \multicolumn{3}{c|}{R\upshape{es}N\upshape{et}56} & \multicolumn{3}{c}{WRN-16-8}\\
\cline{2-7}
 & ACC (\%) & NLL & B\upshape{rier} (\textperthousand) & ACC (\%) & NLL & B\upshape{rier} (\textperthousand) \\
\hline
\hline
\upshape{cyc}SGHMC$\times$T4  & 81.23$\pm$0.19 & 6770$\pm$59 & 3.18$\pm$0.08   &  82.98$\pm$0.03 & 6384$\pm$11 & 2.17$\pm$0.05  \\ 
\upshape{cyc}SWAG$\times$T4                & 81.14$\pm$0.11 & 6744$\pm$55 & 3.06$\pm$0.09   &  83.05$\pm$0.04 & 6359$\pm$14 & 2.04$\pm$0.07  \\ 
\hline
\hline
M-SGD$\times$P4 & 81.03$\pm$0.14   & 6847$\pm$22  & \textbf{2.86$\pm$0.08}   & 82.57$\pm$0.07  & 6821$\pm$21 & \textbf{1.77$\pm$0.06} \\ 
\upshape{re}SGHMC$\times$P4 & 81.11$\pm$0.16   & 6915$\pm$40  & 2.92$\pm$0.12   & 82.72$\pm$0.08  & 6452$\pm$19 & 1.92$\pm$0.04  \\ 
ICSGHMC$\times$P4 & \textbf{81.51$\pm$0.18}  & \textbf{6630$\pm$38}   & 2.88$\pm$0.09  & \textbf{83.12$\pm$0.10}   & \textbf{6338$\pm$36} & 1.83$\pm$0.06  \\ 
\hline
\end{tabular}
\end{center}
\end{sc}
\vspace{-0.15in}
\end{table*}

Table \ref{UQ_test} shows that the vanilla ensemble results via M-SGD$\times$P4 surprisingly outperform cycSGHMC$\times$T4 and cycSWAG$\times$T4 on medium models, such as ResNet20 and ResNet32, and show very good performance on the out-of-distribution samples in Brier scores. We suspect that the parallel implementation  ($\times$P4) provides isolated initializations with less correlated samples; by contrast, cycSGHMC$\times$T4 and cycSWAG$\times$T4 explore the energy landscape contiguously, implying a risk to stay near the original region. reSGHMC$\times$P4 shows a remarkable performance overall, but demonstrates a large variance occasionally; this indicates the insufficiency of the swaps when multiple processes are included. When it comes to testing WRN-16-8, cycSWAG$\times$T4 shows a marvelous result and a large improvement compared to the other baselines. We conjecture that cycSWAG is more independent of hyperparameter tuning, thus leading to better performance in larger models. We don't report CSGHMC$\times$P4 since it becomes quite unstable during the training of ResNet56 and WRN-16-8 models and causes mediocre results. As to ICSGHMC$\times$P4, it consistently performs remarkable in both ACC and NLL and performs comparable to M-SGD$\times$P4 in Brier scores. 

Code is available at
\href{https://github.com/WayneDW/Interacting-Contour-Stochastic-Gradient-Langevin-Dynamics}{\footnotesize{github.com/WayneDW/Interacting-Contour-Stochastic-Gradient-Langevin-Dynamics}}.

\section{Conclusion}
\vskip -0.05in
We have proposed the ICSGLD as an efficient algorithm for sampling from distributions with a complex energy landscape, and shown theoretically that ICSGLD is indeed more efficient than the single-chain CSGLD for a slowly decreasing step size. To our best knowledge, this is the first interacting importance sampling algorithm that adapts to big data problems without scalability concerns. ICSGLD has been compared with numerous state-of-the-art baselines for various tasks, whose remarkable results indicate its promising future in big data applications.

\section*{Acknowledgment} 

Liang's research was supported in part by the grants DMS-2015498, R01-GM117597 and R01-GM126089. Lin acknowledges the support from NSF (DMS-1555072, DMS-2053746, and DMS-2134209), BNL Subcontract 382247, and DE-SC0021142.

\bibliography{mybib2}

\begin{thebibliography}{62}
\providecommand{\natexlab}[1]{#1}
\providecommand{\url}[1]{\texttt{#1}}
\expandafter\ifx\csname urlstyle\endcsname\relax
  \providecommand{\doi}[1]{doi: #1}\else
  \providecommand{\doi}{doi: \begingroup \urlstyle{rm}\Url}\fi

\bibitem[Ahn et~al.(2012)Ahn, Korattikara, and Welling]{Ahn12}
Sungjin Ahn, Anoop Korattikara, and Max Welling.
\newblock {B}ayesian {P}osterior {S}ampling via {S}tochastic {G}radient
  {F}isher {S}coring.
\newblock In \emph{Proc. of the International Conference on Machine Learning
  (ICML)}, 2012.

\bibitem[Ahn et~al.(2014)Ahn, Shahbaba, and Welling]{Ahn14_icml}
Sungjin Ahn, Babak Shahbaba, and Max Welling.
\newblock {D}istributed {S}tochastic {G}radient {MCMC}.
\newblock In \emph{Proc. of the International Conference on Machine Learning
  (ICML)}, 2014.

\bibitem[Aitchison(2021)]{Aitchison2021}
Laurence Aitchison.
\newblock A {S}tatistical {T}heory of {C}old {P}osteriors in {D}eep {N}eural
  {N}etworks.
\newblock In \emph{Proc. of the International Conference on Learning
  Representation (ICLR)}, 2021.

\bibitem[Andrieu et~al.(2005)Andrieu, Moulines, and Priouret]{AndrieuMP2005}
C.~Andrieu, E.~Moulines, and P.~Priouret.
\newblock Stability of {S}tochastic {A}pproximation under {V}erifiable
  {C}onditions.
\newblock \emph{SIAM J. Control Optim.}, 44\penalty0 (1):\penalty0 283--312,
  2005.

\bibitem[Andrieu et~al.(2010)Andrieu, Doucet, and Holenstein]{Particle_MCMC}
Christophe Andrieu, Arnaud Doucet, and Roman Holenstein.
\newblock Particle {M}arkov {C}hain {M}onte {C}arlo {M}ethods.
\newblock \emph{Journal of the Royal Statistical Society: Series B (Statistical
  Methodology)}, 72\penalty0 (3), 2010.

\bibitem[Benveniste et~al.(1990)Benveniste, M\'etivier, and Priouret]{Albert90}
Albert Benveniste, Michael M\'etivier, and Pierre Priouret.
\newblock \emph{Adaptive {A}lgorithms and {S}tochastic {A}pproximations}.
\newblock Berlin: Springer, 1990.

\bibitem[Bittner et~al.(2008)Bittner, Nussbaumer, and Janke]{Elmar08}
Elmar Bittner, Andreas Nussbaumer, and Wolfhard Janke.
\newblock Make {L}ife {S}imple: {U}nleash the {F}ull {P}ower of the {P}arallel
  {T}empering {A}lgorithm.
\newblock \emph{Physical Review Letters}, 101:\penalty0 130603--130603, 2008.

\bibitem[Chen et~al.(2015)Chen, Ding, and Carin]{Chen15}
Changyou Chen, Nan Ding, and Lawrence Carin.
\newblock On the {C}onvergence of {S}tochastic {G}radient {MCMC} {A}lgorithms
  with {H}igh-order {I}ntegrators.
\newblock In \emph{Advances in Neural Information Processing Systems
  (NeurIPS)}, pp.\  2278--2286, 2015.

\bibitem[Chen et~al.(2016)Chen, Ding, Li, Zhang, and Carin]{chen16_distributed}
Changyou Chen, Nan Ding, Chunyuan Li, Yizhe Zhang, and Lawrence Carin.
\newblock Stochastic {G}radient {MCMC} with {S}tale {G}radients.
\newblock In \emph{Advances in Neural Information Processing Systems
  (NeurIPS)}, 2016.

\bibitem[Chen et~al.(2014)Chen, Fox, and Guestrin]{Chen14}
Tianqi Chen, Emily~B. Fox, and Carlos Guestrin.
\newblock Stochastic {G}radient {H}amiltonian {M}onte {C}arlo.
\newblock In \emph{Proc. of the International Conference on Machine Learning
  (ICML)}, 2014.

\bibitem[Deng et~al.(2020{\natexlab{a}})Deng, Feng, Gao, Liang, and
  Lin]{deng2020}
Wei Deng, Qi~Feng, Liyao Gao, Faming Liang, and Guang Lin.
\newblock Non-{C}onvex {L}earning via {R}eplica {E}xchange {S}tochastic
  {G}radient {MCMC}.
\newblock In \emph{Proc. of the International Conference on Machine Learning
  (ICML)}, 2020{\natexlab{a}}.

\bibitem[Deng et~al.(2020{\natexlab{b}})Deng, Lin, and Liang]{CSGLD}
Wei Deng, Guang Lin, and Faming Liang.
\newblock A {C}ontour {S}tochastic {G}radient {L}angevin {D}ynamics {A}lgorithm
  for {S}imulations of {M}ulti-modal {D}istributions.
\newblock In \emph{Advances in Neural Information Processing Systems
  (NeurIPS)}, 2020{\natexlab{b}}.

\bibitem[Deng et~al.(2021{\natexlab{a}})Deng, Feng, Karagiannis, Lin, and
  Liang]{deng_VR}
Wei Deng, Qi~Feng, Georgios Karagiannis, Guang Lin, and Faming Liang.
\newblock Accelerating {C}onvergence of {R}eplica {E}xchange {S}tochastic
  {G}radient {MCMC} via {V}ariance {R}eduction.
\newblock In \emph{Proc. of the International Conference on Learning
  Representation (ICLR)}, 2021{\natexlab{a}}.

\bibitem[Deng et~al.(2021{\natexlab{b}})Deng, Ma, Song, Zhang, and Lin]{FA-LD}
Wei Deng, Yi-An Ma, Zhao Song, Qian Zhang, and Guang Lin.
\newblock On {C}onvergence of {F}ederated {A}veraging {L}angevin {D}ynamics.
\newblock \emph{arXiv:2112.05120v1}, 2021{\natexlab{b}}.

\bibitem[Ding et~al.(2014)Ding, Fang, Babbush, Chen, Skeel, and Neven]{Ding14}
Nan Ding, Youhan Fang, Ryan Babbush, Changyou Chen, Robert~D. Skeel, and
  Hartmut Neven.
\newblock Bayesian {S}ampling using {S}tochastic {G}radient {T}hermostats.
\newblock In \emph{Advances in Neural Information Processing Systems
  (NeurIPS)}, pp.\  3203--3211, 2014.

\bibitem[Doucet et~al.(2001)Doucet, de~Freitas, and Gordon]{SMC2}
Arnaud Doucet, Nando de~Freitas, and Neil Gordon.
\newblock \emph{Sequential {M}onte {C}arlo {M}ethods in {P}ractice}.
\newblock Springer Science \& Business Media, 2001.

\bibitem[Durmus \& \'{E}ric Moulines(2017)Durmus and \'{E}ric
  Moulines]{Alain17}
Alain Durmus and \'{E}ric Moulines.
\newblock Non-asymptotic {C}onvergence {A}nalysis for the {U}nadjusted
  {L}angevin {A}lgorithm.
\newblock \emph{Annals of Applied Probability}, 27:\penalty0 1551--1587, 2017.

\bibitem[Earl \& Deem(2005)Earl and Deem]{parallel_tempering05}
David~J. Earl and Michael~W. Deem.
\newblock Parallel {T}empering: {T}heory, {A}pplications, and {N}ew
  {P}erspectives.
\newblock \emph{Phys. Chem. Chem. Phys.}, 7:\penalty0 3910--3916, 2005.

\bibitem[Erdogdu et~al.(2018)Erdogdu, Mackey, and Shamir]{Mackey18}
Murat~A Erdogdu, Lester Mackey, and Ohad Shamir.
\newblock Global {N}on-convex {O}ptimization with {D}iscretized {D}iffusions.
\newblock In \emph{Advances in Neural Information Processing Systems
  (NeurIPS)}, 2018.

\bibitem[Fort et~al.(2011)Fort, Moulines, and Priouret]{Fortetal2011}
G.~Fort, E.~Moulines, and P.~Priouret.
\newblock Convergence of {A}daptive and {I}nteracting {M}arkov {C}hain {M}onte
  {C}arlo {A}lgorithms.
\newblock \emph{Annals of Statistics}, 39:\penalty0 3262--3289, 2011.

\bibitem[Fort et~al.(2015)Fort, Jourdain, Kuhn, Lelièvre, and Stoltz]{Fort15}
G.~Fort, B.~Jourdain, E.~Kuhn, T.~Lelièvre, and G.~Stoltz.
\newblock Convergence of the {W}ang-{L}andau {A}lgorithm.
\newblock \emph{Math. Comput.}, 84\penalty0 (295):\penalty0 2297–2327, 2015.

\bibitem[Futami et~al.(2020)Futami, Sato, and Sugiyama]{Futoshi2020}
Futoshi Futami, Issei Sato, and Masashi Sugiyama.
\newblock Accelerating the {D}iffusion-based {E}nsemble {S}ampling by
  {N}on-reversible {D}ynamics.
\newblock In \emph{Proc. of the International Conference on Machine Learning
  (ICML)}, 2020.

\bibitem[Gal \& Ghahramani(2016)Gal and Ghahramani]{Gal16b}
Yarin Gal and Zoubin Ghahramani.
\newblock Dropout as a {B}ayesian {A}pproximation: {R}epresenting {M}odel
  {U}ncertainty in {D}eep {L}earning.
\newblock In \emph{Proc. of the International Conference on Machine Learning
  (ICML)}, 2016.

\bibitem[Geyer(1991)]{Geyer91}
Charles~J. Geyer.
\newblock Markov {C}hain {M}onte {C}arlo {M}aximum {L}ikelihood.
\newblock \emph{Computing Science and Statistics: Proceedings of the 23rd
  Symposium on the Interfac}, pp.\  156--163, 1991.

\bibitem[Gordon et~al.(1993)Gordon, Salmond, and Smith]{SMC1}
Neil~J Gordon, David~J Salmond, and Adrian~FM Smith.
\newblock Novel {A}pproach to {N}onlinear/{N}on-{G}aussian {B}ayesian {S}tate
  {E}stimation.
\newblock \emph{IEE Proceedings F (Radar and Signal Processing)}, 140\penalty0
  (2), 1993.

\bibitem[He et~al.(2016)He, Zhang, Ren, and Sun]{kaiming15}
Kaiming He, Xiangyu Zhang, Shaoqing Ren, and Jian Sun.
\newblock Deep {R}esidual {L}earning for {I}mage {R}ecognition.
\newblock In \emph{The IEEE Conference on Computer Vision and Pattern
  Recognition (CVPR)}, 2016.

\bibitem[Izmailov et~al.(2018)Izmailov, Podoprikhin, Garipov, Vetrov, and
  Wilson]{SWA1}
Pavel Izmailov, Dmitry Podoprikhin, Timur Garipov, Dmitry Vetrov, and
  Andrew~Gordon Wilson.
\newblock Averaging {W}eights {L}eads to {W}ider {O}ptima and {B}etter
  {G}eneralization.
\newblock In \emph{Proc. of the Conference on Uncertainty in Artificial
  Intelligence (UAI)}, 2018.

\bibitem[Jarrett et~al.(2009)Jarrett, Kavukcuoglu, Ranzato, and
  LeCun]{Jarrett09}
K.~Jarrett, K.~Kavukcuoglu, M.~Ranzato, and Y.~LeCun.
\newblock What is the {B}est {M}ulti-stage {A}rchitecture for {O}bject
  {R}ecognition?
\newblock In \emph{Proc. of the International Conference on Computer Vision
  (ICCV)}, pp.\  2146--2153, September 2009.

\bibitem[Katzgraber et~al.(2008)Katzgraber, Trebst, Huse, and
  Troyer]{Katzgraber06}
Helmut~G Katzgraber, Simon Trebst, David~A Huse, and Matthias Troyer.
\newblock {F}eedback-{O}ptimized {P}arallel {T}empering {M}onte {C}arlo.
\newblock \emph{Journal of Statistical Mechanics: Theory and Experiment}, pp.\
  p. P03018, 2008.

\bibitem[Li et~al.(2016)Li, Chen, Carlson, and Carin]{Li16}
Chunyuan Li, Changyou Chen, David Carlson, and Lawrence Carin.
\newblock Preconditioned {S}tochastic {G}radient {L}angevin {D}ynamics for
  {D}eep {N}eural {N}etworks.
\newblock In \emph{Proc. of the National Conference on Artificial Intelligence
  (AAAI)}, pp.\  1788--1794, 2016.

\bibitem[Li et~al.(2019{\natexlab{a}})Li, Chen, Pu, Henao, and Carin]{Li19_v2}
Chunyuan Li, Changyou Chen, Yunchen Pu, Ricardo Henao, and Lawrence Carin.
\newblock Communication-{E}fficient {S}tochastic {G}radient {MCMC} for {N}eural
  {N}etworks.
\newblock In \emph{Proc. of the National Conference on Artificial Intelligence
  (AAAI)}, 2019{\natexlab{a}}.

\bibitem[Li et~al.(2020)Li, Huang, Yang, Wang, and Zhang]{lhy+20}
Xiang Li, Kaixuan Huang, Wenhao Yang, Shusen Wang, and Zhihua Zhang.
\newblock On the {C}onvergence of {F}ed{A}vg on {N}on-{IID} {D}ata.
\newblock In \emph{Proc. of the International Conference on Learning
  Representation (ICLR)}, 2020.

\bibitem[Li et~al.(2019{\natexlab{b}})Li, Wu, Mackey, and Erdogdu]{Li19}
Xuechen Li, Denny Wu, Lester Mackey, and Murat~A. Erdogdu.
\newblock Stochastic {R}unge-{K}utta {A}ccelerates {L}angevin {M}onte {C}arlo
  and {B}eyond.
\newblock In \emph{Advances in Neural Information Processing Systems
  (NeurIPS)}, pp.\  7746--7758, 2019{\natexlab{b}}.

\bibitem[Liang et~al.(2007)Liang, Liu, and Carroll]{Liang07}
Faming Liang, Chuanhai Liu, and Raymond~J. Carroll.
\newblock {S}tochastic {A}pproximation in {M}onte {C}arlo {C}omputation.
\newblock \emph{Journal of the American Statistical Association}, 102:\penalty0
  305--320, 2007.

\bibitem[Liu \& Wang(2016)Liu and Wang]{SVGD}
Qiang Liu and Dilin Wang.
\newblock Stein {V}ariational {G}radient {D}escent: {A} {G}eneral {P}urpose
  {B}ayesian {I}nference {A}lgorithm.
\newblock In \emph{Advances in Neural Information Processing Systems
  (NeurIPS)}, 2016.

\bibitem[Maddox et~al.(2019)Maddox, Garipov, Izmailov, Vetrov, and
  Wilson]{swag}
Wesley Maddox, Timur Garipov, Pavel Izmailov, Dmitry Vetrov, and Andrew~Gordon
  Wilson.
\newblock A {S}imple {B}aseline for {B}ayesian {U}ncertainty in {D}eep
  {L}earning.
\newblock In \emph{Advances in Neural Information Processing Systems
  (NeurIPS)}, 2019.

\bibitem[Mandt et~al.(2017)Mandt, Hoffman, and Blei]{Mandt}
Stephan Mandt, Matthew~D. Hoffman, and David~M. Blei.
\newblock Stochastic {G}radient {D}escent as {A}pproximate {B}ayesian
  {I}nference.
\newblock \emph{Journal of Machine Learning Research}, 18:\penalty0 1--35,
  2017.

\bibitem[Mattingly et~al.(2002)Mattingly, Stuartb, and Highamc]{mattingly02}
J.C. Mattingly, A.M. Stuartb, and D.J. Highamc.
\newblock Ergodicity for {SDE}s and {A}pproximations: {L}ocally {L}ipschitz
  {V}ector {F}ields and {D}egenerate {N}oise.
\newblock \emph{Stochastic Processes and their Applications}, 101:\penalty0
  185--232, 2002.

\bibitem[Mattingly et~al.(2010)Mattingly, Stuart, and Tretyakov]{mattingly10}
Jonathan~C. Mattingly, Andrew~M. Stuart, and M.V. Tretyakov.
\newblock Convergence of {N}umerical {T}ime-{A}veraging and {S}tationary
  {M}easures via {P}oisson {E}quations.
\newblock \emph{SIAM Journal on Numerical Analysis}, 48:\penalty0 552--577,
  2010.

\bibitem[Pelletier(1998)]{Pelletier98}
Mariane Pelletier.
\newblock Weak {C}onvergence {R}ates for {S}tochastic {A}pproximation with
  {A}pplication to {M}ultiple {T}argets and {S}imulated {A}nnealing.
\newblock \emph{Annals of Applied Probability}, 8:\penalty0 10--44, 1998.

\bibitem[Raginsky et~al.(2017)Raginsky, Rakhlin, and Telgarsky]{Maxim17}
Maxim Raginsky, Alexander Rakhlin, and Matus Telgarsky.
\newblock Non-convex {L}earning via {S}tochastic {G}radient {L}angevin
  {D}ynamics: a {N}onasymptotic {A}nalysis.
\newblock In \emph{Proc. of Conference on Learning Theory (COLT)}, June 2017.

\bibitem[Rainforth et~al.(2016)Rainforth, Naesseth, Lindsten, Paige, van~de
  Meent, Doucet, and Wood]{IPMCMC}
Tom Rainforth, Christian~A. Naesseth, Fredrik Lindsten, Brooks Paige,
  Jan-Willem van~de Meent, Arnaud Doucet, and Frank Wood.
\newblock Interacting {P}article {M}arkov {C}hain {M}onte {C}arlo.
\newblock In \emph{Proc. of the International Conference on Machine Learning
  (ICML)}, 2016.

\bibitem[Riquelme et~al.(2018)Riquelme, Tucker, and Snoek]{bandits_showdown}
Carlos Riquelme, George Tucker, and Jasper Snoek.
\newblock Deep {B}ayesian {B}andits {S}howdown.
\newblock In \emph{Proc. of the International Conference on Learning
  Representation (ICLR)}, 2018.

\bibitem[Robbins \& Monro(1951)Robbins and Monro]{RobbinsM1951}
Herbert Robbins and Sutton Monro.
\newblock A {S}tochastic {A}pproximation {M}ethod.
\newblock \emph{Annals of Mathematical Statistics}, 22:\penalty0 400--407,
  1951.

\bibitem[Roberts \& Tweedie(1996)Roberts and
  Tweedie]{Roberts_Tweedie_Bernoulli}
Gareth~O. Roberts and Richard~L. Tweedie.
\newblock Exponential {C}onvergence of {L}angevin {D}istributions and {T}heir
  {D}iscrete {A}pproximations.
\newblock \emph{Bernoulli}, 2\penalty0 (4):\penalty0 341--363, 1996.

\bibitem[Sato \& Nakagawa(2014)Sato and Nakagawa]{Sato2014ApproximationAO}
Issei Sato and Hiroshi Nakagawa.
\newblock Approximation {A}nalysis of {S}tochastic {G}radient {L}angevin
  {D}ynamics by {U}sing {F}okker-{P}lanck {E}quation and {I}to {P}rocess.
\newblock In \emph{Proc. of the International Conference on Machine Learning
  (ICML)}, 2014.

\bibitem[Song et~al.(2014)Song, Wu, and Liang]{SongWL2014}
Qifan Song, Mingqi Wu, and Faming Liang.
\newblock Weak {C}onvergence {R}ates of {P}opulation versus {S}ingle-{C}hain
  {S}tochastic {A}pproximation {MCMC} {A}lgorithms.
\newblock \emph{Advances in Applied Probability}, 46:\penalty0 1059--1083,
  2014.

\bibitem[Swendsen \& Wang(1986)Swendsen and Wang]{PhysRevLett86}
Robert~H. Swendsen and Jian-Sheng Wang.
\newblock Replica {M}onte {C}arlo {S}imulation of {S}pin-{G}lasses.
\newblock \emph{Physical Review Letters}, 57:\penalty0 2607--2609, 1986.

\bibitem[Syed et~al.(2021)Syed, Bouchard-C\^{o}t\'{e}, Deligiannidis, and
  Doucet]{Doucet19}
Saifuddin Syed, Alexandre Bouchard-C\^{o}t\'{e}, George Deligiannidis, and
  Arnaud Doucet.
\newblock Non-{R}eversible {P}arallel {T}empering: a {S}calable {H}ighly
  {P}arallel {MCMC} {s}cheme.
\newblock \emph{Journal of Royal Statistical Society, Series B}, 2021.

\bibitem[Teh et~al.(2016)Teh, Thi{\' e}ry, and Vollmer]{Teh16}
Yee~Whye Teh, Alexandre Thi{\' e}ry, and Sebastian Vollmer.
\newblock Consistency and {F}luctuations for {S}tochastic {G}radient {L}angevin
  {D}ynamics.
\newblock \emph{Journal of Machine Learning Research}, 17:\penalty0 1--33,
  2016.

\bibitem[Vanden-Eijnden(2001)]{Eric}
Eric Vanden-Eijnden.
\newblock Introduction to {R}egular {P}erturbation {T}heory.
\newblock \emph{Slides}, 2001.
\newblock URL \url{https://cims.nyu.edu/~eve2/reg_pert.pdf}.

\bibitem[Vollmer et~al.(2016)Vollmer, Zygalakis, and Teh]{VollmerZW2016}
Sebastian~J. Vollmer, Konstantinos~C. Zygalakis, and Yee~Whye Teh.
\newblock Exploration of the ({N}on-) {A}symptotic {B}ias and {V}ariance of
  {S}tochastic {G}radient {L}angevin {D}ynamics.
\newblock \emph{Journal of Machine Learning Research}, 17\penalty0
  (159):\penalty0 1--48, 2016.

\bibitem[Wang \& Landau(2001)Wang and Landau]{WangLandau2001}
Fugao Wang and David~P. Landau.
\newblock Efficient, {M}ultiple-range {R}andom {W}alk {A}lgorithm to
  {C}alculate the {D}ensity of {S}tates.
\newblock \emph{Physical Review Letters}, 86:\penalty0 2050--3, 2001.

\bibitem[Weinhart et~al.(2010)Weinhart, Singh, and Thornton]{Perturbation2}
T.~Weinhart, A.~Singh, and A.R. Thornton.
\newblock Perturbation {T}heory \& {S}tability {A}nalysis.
\newblock \emph{Slides}, 2010.

\bibitem[Welling \& Teh(2011)Welling and Teh]{Welling11}
Max Welling and Yee~Whye Teh.
\newblock {B}ayesian {L}earning via {S}tochastic {G}radient {L}angevin
  {D}ynamics.
\newblock In \emph{Proc. of the International Conference on Machine Learning
  (ICML)}, pp.\  681--688, 2011.

\bibitem[Wenzel et~al.(2020)Wenzel, Roth, Veeling, \'{S}wiatkowski, Tran,
  Mandt, Snoek, Salimans, Jenatton, and Nowozin]{Florian2020}
Florian Wenzel, Kevin Roth, Bastiaan~S. Veeling, Jakub \'{S}wiatkowski, Linh
  Tran, Stephan Mandt, Jasper Snoek, Tim Salimans, Rodolphe Jenatton, and
  Sebastian Nowozin.
\newblock How {G}ood is the {B}ayes {P}osterior in {D}eep {N}eural {N}etworks
  {R}eally?
\newblock In \emph{Proc. of the International Conference on Machine Learning
  (ICML)}, 2020.

\bibitem[Xu et~al.(2018)Xu, Chen, Zou, and Gu]{Xu18}
Pan Xu, Jinghui Chen, Difan Zou, and Quanquan Gu.
\newblock Global {C}onvergence of {L}angevin {D}ynamics {B}ased {A}lgorithms
  for {N}onconvex {O}ptimization.
\newblock In \emph{Advances in Neural Information Processing Systems
  (NeurIPS)}, 2018.

\bibitem[Zagoruyko \& Komodakis(2016)Zagoruyko and Komodakis]{wide_residual}
Sergey Zagoruyko and Nikos Komodakis.
\newblock Wide {R}esidual {N}etworks.
\newblock In \emph{Proceedings of the British Machine Vision Conference
  (BMVC)}, pp.\  87.1--87.12, September 2016.

\bibitem[Zhang et~al.(2020{\natexlab{a}})Zhang, Zhang, Carin, and Chen]{SPOS}
Jianyi Zhang, Ruiyi Zhang, Lawrence Carin, and Changyou Chen.
\newblock Stochastic {P}article-{O}ptimization {S}ampling and the
  {N}on-{A}symptotic {C}onvergence {T}heory.
\newblock In \emph{Proceedings of the International Workshop on Artificial
  Intelligence and Statistics}, 2020{\natexlab{a}}.

\bibitem[Zhang et~al.(2020{\natexlab{b}})Zhang, Li, Zhang, Chen, and
  Wilson]{ruqi2020}
Ruqi Zhang, Chunyuan Li, Jianyi Zhang, Changyou Chen, and Andrew~Gordon Wilson.
\newblock Cyclical {S}tochastic {G}radient {MCMC} for {B}ayesian {D}eep
  {L}earning.
\newblock In \emph{Proc. of the International Conference on Learning
  Representation (ICLR)}, 2020{\natexlab{b}}.

\bibitem[Zhang et~al.(2017)Zhang, Liang, and Charikar]{Yuchen17}
Yuchen Zhang, Percy Liang, and Moses Charikar.
\newblock A {H}itting {T}ime {A}nalysis of {S}tochastic {G}radient {L}angevin
  {D}ynamics.
\newblock In \emph{Proc. of Conference on Learning Theory (COLT)}, pp.\
  1980--2022, 2017.

\bibitem[Zhong et~al.(2017)Zhong, Zheng, Kang, Li, and Yang]{Zhong17}
Zhun Zhong, Liang Zheng, Guoliang Kang, Shaozi Li, and Yi~Yang.
\newblock Random {E}rasing {D}ata {A}ugmentation.
\newblock \emph{ArXiv e-prints}, 2017.

\end{thebibliography}
\bibliographystyle{iclr2022_conference}

\newpage
\appendix

We summarize the supplementary material as follows: Section \ref{review} provides the preliminary knowledge for stochastic approximation; Section \ref{convergence} shows a local stability condition that adapts to high losses; Section \ref{Gaussian_approx} proves the main asymptotic normality for the stochastic approximation process, which naturally yields the conclusion that interacting contour stochastic gradient Langevin dynamics (ICSGLD) is more efficient than the analogous single chain based on slowly decreasing step sizes; Section \ref{details_exp} details the experimental settings.

\section{Preliminaries}
\label{review}

\subsection{Stochastic approximation}

Given a random-field function $\widetilde H(\bm{\btheta}, \bm{\bx})$, the stochastic approximation algorithm \citep{Albert90} proposes to solve the mean-field equation $h(\btheta)=0$ in the analysis of adaptive algorithms
\begin{equation*}
\begin{split}
\label{sa00}
h(\btheta)&=\int_{\MX} \widetilde H(\bm{\theta}, \bm{\bx}) \varpi_{\bm{\theta}}(d\bm{\bx})=0,
\end{split}
\end{equation*}
where $\bx\in \MX \subset \mathbb{R}^d$, $\btheta\in\bTheta \subset \mathbb{R}^{m}$, $\varpi_{\btheta}(\bx)$ is a distribution that depends on the self-adapting parameter $\btheta$. Given the transition kernel $\Pi_{\bm{\theta}}(\bm{x}, A)$ for any Borel subset $A\subset \MX$, the algorithm can be written as follows
\begin{itemize}
\item[(1)] Simulate $\bx_{k+1}\sim\Pi_{\bm{\theta_{k}}}(\bx_{k}, \cdot)$, which yields the invariant distribution $ \varpi_{\bm{\theta_k}}(\cdot)$, 

\item[(2)] Optimize $\bm{\theta}_{k+1}=\bm{\theta}_{k}+\omega_{k+1} \widetilde H(\bm{\theta}_{k}, \bx_{k+1}).$
\end{itemize}
Compared with the standard Robbins–Monro algorithm \citep{RobbinsM1951}, the algorithm proposes to simulate $\bx$ from a transition kernel $\Pi_{\bm{\theta}}(\cdot, \cdot)$ instead of the distribution $\varpi_{\bm{\theta}}(\cdot)$ directly. In other words, , $\widetilde H(\btheta_k, \bx_{k+1})-h(\btheta_k)$ is not a Martingale but rather a Markov state-dependent noise.

\subsection{Poisson's equation}

In the stochastic approximation algorithm, the sequence of $\{(\bx_k, \btheta_k)\}_{k=1}^{\infty}$ on the product space $\MX\times \bTheta$ is generated, which is an inhomogeneous Markov chain and requires the tool of the Poisson's equation to study the convergence
\begin{equation*}
    \mu_{\btheta}(\bm{x})-\mathrm{\Pi}_{\bm{\theta}}\mu_{\bm{\theta}}(\bm{x})=\widetilde H(\bm{\theta}, \bm{x})-h(\bm{\theta}),
\end{equation*}
where $\mu_{\btheta}(\cdot)$ is a function on $\MX$. 
The solution $\mu_{\btheta}(\bm{x})$ to the Poisson's equation exists and is formulated in the form
\begin{equation*}
    \mu_{\btheta}(\bx):=\sum_{k\geq 0} \Pi_{\btheta}^k (\widetilde H(\btheta, \bx)-h(\btheta)),
\end{equation*}
when the above series converges. To ensure such a convergence, \citet{Albert90} made the following regularity conditions on the solution $\mu_{\btheta}(\cdot)$ of the Poisson's equation:

{\it There exist a Lyapunov function $V: \MX \to [1,\infty)$ and a positive constant $C>0$ such that  $\forall\bm{\theta}, \bm{\theta}'\in \bm{\bTheta}$}, we have
\begin{equation}
\begin{split}
\label{Lyapunov_condition}
\|\mathrm{\Pi}_{\bm{\theta}}\mu_{\btheta}(\bx)\|&\leq C V(\bx),\quad
\|\mathrm{\Pi}_{\bm{\theta}}\mu_{\bm{\theta}}(\bx)
-\mathrm{\Pi}_{\bm{\theta'}}\mu_{\bm{\theta'}}(\bx)\|\leq C\|\bm{\theta}
-\bm{\theta}'\| V(\bx),  \quad 
\E[V(\bx)]\leq \infty,\\
\end{split}
\end{equation}
where a common choice for the Lyapunov function is to set $V(\bx)=1+\|\bx\|^2$ \citep{Teh16, VollmerZW2016}.

\subsection{Gaussian diffusions}

Consider a stochastic linear differential equation
\begin{equation}
\label{slde}
    d\bU_t=h_{\btheta}(\btheta_t) \bU_t dt + \bR^{1/2}(\btheta_t)d\bW_t,
\end{equation}
where $\bU$ is a $m$-dimensional random vector, $h_{\btheta}:=\frac{d}{d\btheta} h(\btheta)$, $\bR(\btheta):=\sum_{k=-\infty}^{\infty} \cov_{\btheta}(H(\btheta, \bx_k), H(\btheta, \bx_0))$ is a positive definite matrix that depends on $\btheta(\cdot)$, $\bW\in\mathbb{R}^m$ is a standard Brownian motion. Given a large enough $t$ such that $\btheta_t$ converges to a fixed point $\widehat\btheta_{\star}$ sufficiently fast, we may write the diffusion associated with Eq.(\ref{slde}) as follow
\begin{equation}
\label{solution_slde}
    \bU_t\approx e^{-th_{\btheta}(\widehat\btheta_{\star})}\bU_0+\int_0^t e^{-(t-s)h_{\btheta}(\widehat\btheta_{\star})}\circ \bR(\widehat\btheta_{\star}) d\bW_s,
\end{equation}

Suppose that the matrix $h_{\btheta}(\widehat\btheta_{\star})$ is negative definite, then $\bU_t$ converges in distribution to a Gaussian variable 
\begin{equation*}
\begin{split}
    \E[\bU_t]&=e^{-t h_{\btheta}(\widehat\btheta_{\star})}\bU_0\\
\var(\bU_t)&=\int_0^{t} e^{t h_{\btheta}(\widehat\btheta_{\star})}\circ \bR \circ e^{t h_{\btheta}(\widehat\btheta_{\star})} du.
\end{split}
\end{equation*}

The main goal of this supplementary file is to study the Gaussian approximation of the process $\omega_k^{-1/2}(\btheta_k-\widehat\btheta_{\star})$ to the solution Eq.(\ref{solution_slde}) for a proper step size $\omega_k$. Thereafter, the advantage of interacting mechanisms can be naturally derived.

\section{Stability and convergence analysis}
\label{convergence}

As required by the algorithm, we update $P$ contour stochastic gradient Langevin dynamics (CSGLD) simultaneously. For the notations, we denote the particle of the p-th chain at iteration $k$ by $\bx_{k}^{(p)}\in \MX\subset \mathbb{R}^d$ and the joint state of the $P$ parallel particles at iteration $k$ by $\bx_{k}^{\pop P}:=\left(\bx_{k}^{(1)}, \bx_{k}^{(2)}, \cdots, \bx_{k}^{(P)}\right)^\top\in \MX^{\pop P}\subset \mathbb{R}^{dP}$. We also denote the learning rate and step size at iteration $k$ by $\epsilon_k$ and $\omega_k$, respectively. We denote by  $\mathcal{N}({0, \bm{I}_{dP}})$ a standard $dP$-dimensional Gaussian vector and denote by $\zeta$ a positive hyperparameter.

\subsection{ICSGLD algorithm} \label{Alg:app}

First, we introduce the interacting contour stochastic gradient Langevin dynamics (ICSGLD) with $P$ parallel chains:
\begin{itemize}
\item[(1)] Simulate $\bx_{k+1}^{\pop P}=\bx_k^{\pop P}- \epsilon_k\nabla_{\bx} \widetilde \bL(\bx_k^{\pop P}, \btheta_k)+\mathcal{N}({0, 2\epsilon_k \tau\bm{I}_{dP}}), \ \ \ \ \ \ \ \ \ \ \ \ \ \ \ \ \ \ \ \ \ \ \ \ \ \ \ \ \ \ \ \ \ \ \ \ \ \ \ \ \  (\text{S}_1)$

\item[(2)] Optimize $\bm{\theta}_{k+1}=\bm{\theta}_{k}+\omega_{k+1} \widetilde \bH(\bm{\theta}_{k}, \bx_{k+1}^{\pop P}),
\ \ \ \ \ \ \ \ \ \ \ \ \ \ \ \ \ \ \ \ \ \ \ \ \  \ \ \ \ \ \ \ \ \ \ \ \ \ \ \ \ \ \ \ \ \ \ \ \ \ \ \ \ \ \ \ \ \ \ \ \ \ \ \ \ \ \ \ \ \ \ \ \ \ \ \ \ (\text{S}_2)$
\end{itemize}
where  $\nabla_{\bx} \widetilde \bL(\bx^{\pop P}, \btheta):=\left(\nabla_{\bx} \widetilde L(\bx^{(1)}, \btheta), \nabla_{\bx} \widetilde L(\bx^{(2)}, \btheta), \cdots, \nabla_{\bx} \widetilde L(\bx^{(P)}, \btheta)\right)^\top$, $\nabla_{\bx} \widetilde L(\bx, \btheta)$ is the 
stochastic adaptive gradient given by 
\begin{equation}
\label{adaptive_grad}
    \nabla_{\bx} \widetilde{L}(\bx,\btheta)= \frac{N}{n} \underbrace{\left[1+ 
   \frac{\zeta\tau}{\Delta u}  \left(\log \theta({J}_{\widetilde U}(\bx))-\log\theta((J_{\widetilde U}(\bx)-1)\vee 1) \right) \right]}_{\text{gradient multiplier}} 
    \nabla_{\bx} \widetilde U(\bx).
\end{equation}

In particular, the interacting random-field function is written as \begin{equation}
\label{interactions}
    \widetilde \bH(\bm{\theta}_{k}, \bx_{k+1}^{\pop P})=\frac{1}{P}\sum_{p=1}^P \widetilde H(\btheta_k,\bx_{k+1}^{(p)}),
\end{equation} 
where each random-field function $\widetilde H(\btheta,\bx)=(\widetilde H_1(\btheta,\bx), \ldots,\widetilde H_m(\btheta,\bx))$ follows
\begin{equation}
\label{random_field_H}
     \widetilde H_i(\btheta,\bx)={\theta}( J_{\widetilde U}(\bx))\left(1_{i= J_{\widetilde U}(\bx)}-{\theta}(i)\right), \quad i=1,2,\ldots,m.
\end{equation}
Here $J_{\widetilde U}(\bx)$ denotes the index $i\in\{1, 2,3,\cdots, m\}$ such that $u_{i-1}< \frac{N}{n} \widetilde U(\bx)\leq u_i$ for a set of energy partitions $\{u_i\}_{i=0}^{m}$ and $\widetilde U(\bx)=\sum_{i\in B} U_i(\bx)$ where $U_i$ denotes the negative log of a posterior based on a single data point $i$ and $B$ denotes a mini-batch of data of size $n$. Note that the stochastic energy estimator $\widetilde U(\bx)$ results in a biased estimation for the partition index $J_{\widetilde U}(\bx)$ due to a non-linear transformation. To avoid such a bias asymptotically with respect to the learning rate $\epsilon_k$, we may consider a variance-reduced energy estimator $\widetilde U_{\text{VR}}(\bx)$ following \citet{deng_VR}
\begin{equation}
\label{VR_estimator}
    \frac{N}{n} \widetilde U_{\text{VR}}(\bx)=\frac{N}{n}\sum_{i\in B_k}\left( U_i(\bx) - U_i\left(\bx_{q\lfloor\frac{k}{q}\rfloor}\right) \right)+\sum_{i=1}^N U_i\left(\bx_{q\lfloor\frac{k}{q}\rfloor}\right),
\end{equation}
where the control variate $\bx_{q\lfloor \frac{k}{q}\rfloor}$ is updated every $q$ iterations.

Compared with the na\"ive parallelism of CSGLD, a key feature of the ICSGLD algorithm lies in the joint estimation of the interacting random-field function $\widetilde \bH(\bm{\theta}, \bx^{\pop P})$ in Eq.(\ref{interactions}) for the same mean-field function $h(\btheta)$.

\subsubsection{Discussions on the hyperparameters}
\label{hyper_setup}
The most important hyperparameter is $\zeta$. A fine-tuned $\zeta$ usually leads to a small or even slightly negative learning rate in low energy regions to avoid local-trap problems. Theoretically, $\zeta$ affects the $L^2$ convergence rate hidden in the big-O notation in Lemma \ref{convex_appendix}.

The other hyperparameters can be easily tuned. For example, the ResNet models yields the full loss ranging from 10,000 to 60,000 after warm-ups, we thus partition the sample space according to the energy into 200 subregions equally without tuning; since the optimization of SA is nearly convex, tuning $\{\omega_k\}$ is much easier than tuning $\{\epsilon_k\}$ for non-convex learning.

\subsubsection{Discussions on distributed computing and communication cost}
\label{communication_cost}
In shared-memory settings, the implementation is trivial and the details are omitted.

In distributed-memory settings: $\btheta_{k+1}$ is updated by the central node as follows: 
\begin{itemize}
    \item The $p$-th worker conducts the sampling step $(\text{S}_1)$ and sends the indices $J_{\widetilde U(\bx_{k+1}^{(p)})}$'s to the central node;
    \item The central node aggregates the indices from all worker and updates $\btheta_k$ based on $(\text{S}_2)$;
    \item The central node sends $\btheta_{k+1}$ back to each worker.
\end{itemize}

We emphasize that we don't communicate the model parameters $\bx\in\mathbb{R}^d$, but rather share the self-adapting parameter $\btheta\in\mathbb{R}^m$, where $m\ll d$.  For example, WRN-16-8 has 11 M parameters (40 MB), while $\btheta$ can be set to dimension $200$ of size 4 KB; hence, the communication cost is not a big issue. Moreover, the theoretical advantage still holds if the communication frequency is slightly reduced.

\subsubsection{Scalability to big data}
\label{scalability}

Recall that the adaptive sampler follows that
\begin{equation*} 
  \footnotesize
  \small{\bx_{k+1}=\bx_k - \epsilon_{k+1} \frac{N}{n} \underbrace{\left[1+ 
   \zeta\tau\frac{\log {\theta}_{k}(J_{\widetilde U}(\bx_k)) - \log{\theta}_{k}((J_{\widetilde U}(\bx_k)-1)\vee 1)}{\Delta u}  \right]}_{\text{gradient multiplier}}  
    \nabla_{\bx} \widetilde U(\bx_k) +\sqrt{2 \tau \epsilon_{k+1}} w_{k+1}}, 
\end{equation*}

The key to the success of (I)CSGLD is to generate sufficiently strong bouncy moves (\emph{negative} gradient multiplier) to escape local traps. To this end, $\zeta$ can be tuned to generate proper bouncy moves. 

Take the CIFAR100 experiments for example: 
\begin{itemize}
    \item the self-adjusting mechanism fails if the gradient multiplier uniformly ``equals'' to 1 and a too small value of $\zeta=1$ could lead to this issue;
    \item the self-adjusting mechanism works only if we choose a large enough $\zeta$ such as 3e6 to  generate (desired) negative gradient multiplier in over-visited regions.
\end{itemize}
However, when we set $\zeta=$3e6, the original stochastic approximation (SA) update proposed in \citep{CSGLD} follows that
$${\theta}_{k+1}(i)={\theta}_{k}(i)+\omega_{k+1}\underbrace{{\theta}_{k}^{\textcolor{red}{\zeta}}(J_{\widetilde U}(\bx_{k+1}))}_{\textbf{essentially 0 for } \zeta\gg 1}\left(1_{i=J_{\widetilde U}(\bx_{k+1})}-{\theta}_{k}(i)\right).$$
Since $\theta(i)<1$ for any $i\in\{1,\cdots, m\}$, $\theta(i)^{\zeta}$ {is essentially 0} for such a large $\zeta$, which means that \textbf{the original SA fails to optimize when $\zeta$ is large}. Therefore, the limited choices of $\zeta$ inevitably limits the scalability to big data problems. Our newly proposed SA scheme $${\theta}_{k+1}(i)={\theta}_{k}(i)+\omega_{k+1}\underbrace{{\theta}_{k}(J_{\widetilde U}(\bx_{k+1}))}_{\text{independent of } \zeta}\left(1_{i=J_{\widetilde U}(\bx_{k+1})}-{\theta}_{k}(i)\right)$$  is more independent of $\zeta$ and proposes to converge to a much smoother equilibrium $\theta_{\infty}^{1/\zeta}$ instead of $\theta_{\infty}$, where $\theta_{\infty}(i)=\int_{\chi_{i}} \pi(x) d x\propto \int_{\chi_{i}} e^{-\frac{U(x)}{\tau}} d x$ is the energy PDF. As such, despite the linear stability is sacrificed, the resulting algorithm is more scalable. For example, estimating $e^{-10,000\times\frac{1}{\zeta}}$ is numerically much easier than $e^{-10,000}$ for a large $\zeta$ such as $10,000$, where $10,000$ can be induced by the high losses in training deep neural networks in big data.

\subsection{Assumptions} 
\label{App:convergence}

A long-standing problem for stochastic approximation is the difficulty in establishing the stability property and a practical remedy for this problem is to study $\bTheta$ on a fixed compact set.
\begin{assump}[Compactness] \label{ass2a} 
The space $\bTheta$ is compact and for any  $i\in \{1,2,\ldots,m\}$ we have $\inf_{\bTheta} \theta(i) >0$. In addition, there exists a positive constant $Q>0$ that satisfies $\forall\btheta\in \bTheta$ and $\bx \in \MX$, 
\begin{equation}
\label{compactness}
     \|\btheta\|\leq Q, \quad  
     \|\widetilde H(\btheta, \bx)\|\leq Q.
\end{equation}
\end{assump}

For weaker assumptions, we refer readers to Theorem 3.2 \citep{Fort15}, where a recurrence property can be proved for the Metropolis-based Wang-Landau algorithm, which eventually established that the estimates return to a desired compact set often enough.

Next, we lay out the smoothness assumption, which is standard in the convergence analysis of SGLD, see e.g. \citet{mattingly10}, \citet{Maxim17} and \citet{Xu18}. 
\begin{assump}[Smoothness]
\label{ass2}
$U(\bm{\xeta})$ is $M$-smooth when there exists a positive constant $M$ that satisfies $\forall\bx, \bx'\in \MX$,
\begin{equation}
\label{ass_2_1_eq}
\begin{split}
\|\nabla_{\bx} U(\bx)-\nabla_{\bx} U(\bm{\bx}')\| & \leq M\|\bx-\bx'\|. \\
\end{split}
\end{equation}
\end{assump}

In addition, we assume the dissipativity condition to ensure that the geometric ergodicity of the dynamical system holds. This assumption is also crucial for verifying the solution properties of the solution of Poisson's equation. Similar assumptions have been made in \citet{mattingly10, Maxim17} and \citet{Xu18}.
\begin{assump}[Dissipativity]
\label{ass3}
 There exist constants $\tilde{m}>0$ and $\tilde{b}\geq 0$ that satisfies $\forall\bx \in \MX$ and any $\btheta \in \bTheta$, 
\label{ass_dissipative}
\begin{equation}
\label{eq:01}
\langle \nabla_{\bx} L(\bx, \btheta), \bx\rangle\leq \tilde{b}-\tilde{m}\|\bx\|^2.
\end{equation}
\end{assump}

To further establish a bounded second moment on $\bx\in \MX$ with respect to a proper Lyapunov function $V(\bx)$, we impose the following conditions 
on the gradient noise:
\begin{assump}[Gradient noise] 
\label{ass4}
The stochastic gradient based on mini-batch settings is an unbiased estimator such that 
\begin{equation*}
\E[\nabla_{\bx}\widetilde U(\bx_{k})-\nabla_{\bx} U(\bx_{k})]=0;
\end{equation*}
furthermore, for some positive constants $M$ and $B$, we have 
\begin{equation*} 
\E [ \|\nabla_{\bx}\widetilde U(\bx_{k})-\nabla_{\bx} U(\bx_{k})\|^2 ] \leq M^2 \|\bx\|^2+B^2,
\end{equation*}
where $\E[\cdot]$ acts on the distribution of the noise in the stochastic gradient $\nabla_{\bx}\widetilde U(\bx_{k})$.
\end{assump}

\subsection{Local stability via the scalable random-field function}
Now, we are ready to present our first result. Lemma \ref{convex_appendix} establishes a local stability condition for the non-linear mean-field system of ICSGLD, which implies a potential convergence of $\btheta_k$ to a unique fixed point that adapts to a wide energy range under mild assumptions. 

\begin{lemma}[Local stability, restatement of Lemma \ref{convex_main}] \label{convex_appendix} 
Assume Assumptions  \ref{ass2a}-\ref{ass4} hold. Given any small enough learning rate $\epsilon$, a large enough $m$ and batch size $n$, and any $\btheta\in \widetilde\bTheta$, where $\widetilde\bTheta$ is a small neighborhood of $\btheta_{\star}$ that contains $\widehat\btheta_{\star}$, we have $\langle h(\btheta), \btheta - \widehat\btheta_{\star}\rangle \leq  -\phi\|\btheta - \widehat\btheta_{\star}\|^2$, where $\widehat\btheta_{\star}=\btheta_{\star}+\mathcal{O}(\varepsilon)$, $\varepsilon=\mathcal{O}\left(\sup_{\bx}\Var(\xi_n(\bx))+\epsilon+\frac{1}{m}\right)$ and  $\btheta_{\star}=\left(\frac{\left(\int_{\MX_1} \pi(\bx)d\bx\right)^{\frac{1}{\zeta}}}{\sum_{k=1}^m \left(\int_{\MX_k} \pi(\bx)d\bx\right)^{\frac{1}{\zeta}}}, 
\ldots,
\frac{\left(\int_{\MX_m} \pi(\bx)d\bx\right)^{\frac{1}{\zeta}}}{\sum_{k=1}^m \left(\int_{\MX_k} \pi(\bx)d\bx\right)^{\frac{1}{\zeta}}}\right)$, $\phi=\inf_{\btheta}\min_i \widehat Z_{\zeta,\theta(i)}^{-1}\big(1-\mathcal{O}(\varepsilon)\big)>0$, $\widehat Z_{\zeta,\theta(i)}$ is defined below Eq.(\ref{h_i_theta_v2}), and $\xi_n(\bx)$ denotes the noise in the energy estimator $\widetilde U(\bx)$ of batch size $n$ and $\Var(\cdot)$ denotes the variance.
\end{lemma}

\begin{proof} The random-field function $\widetilde H_i(\btheta,\bx)={\theta}(J_{\widetilde U}(\bx))\left(1_{i= J_{\widetilde U}(\bx)}-{\theta}(i)\right)$ based on the stochastic energy estimator $\widetilde U(\bx)$ yields a biased estimator of $ H_i(\btheta,\bx)={\theta}( J(\bx))\left(1_{i= J(\bx)}-{\theta}(i)\right)$ for any $i \in \{1,2,\ldots,m\}$ based on the exact energy partition function $J(\cdot)$. By Lemma.\ref{bias_in_SA}, we know that the bias caused by the stochastic energy is of order $\mathcal{O}(\Var(\xi_n(\bx)))$.

Now we compute the mean-field function $h(\btheta)$ based on the measure $\varpi_{\btheta}(\bx)$ simulated from SGLD:
\begin{equation}
\small
\label{iiii}
\begin{split} 
        h_i(\btheta)&=\int_{\MX} \widetilde H_i(\btheta,\bx) 
         \varpi_{\btheta}(\bx) d\bx
         =\int_{\MX} H_i(\btheta,\bx) 
         \varpi_{\btheta}(\bx) d\bx+\mathcal{O}\left(\Var(\xi_n(\bx))\right)\\
         &=\ \int_{\MX} H_i(\btheta,\bx) \left( \underbrace{\varpi_{\widetilde{\Psi}_\btheta}(\bx)}_{\text{I}_1} \underbrace{-\varpi_{\widetilde{\Psi}_\btheta}(\bx)+\varpi_{\Psi_{\btheta}}(\bx)}_{\text{I}_2: \text{piece-wise approximation}}\underbrace{-\varpi_{\Psi_{\btheta}}(\bx)+\varpi_{\btheta}(\bx)}_{\text{I}_3: \text{numerical discretization}}\right) d\bx+\mathcal{O}\left(\Var(\xi_n(\bx))\right),\\
\end{split}
\end{equation}
where $\varpi_{\btheta}$ is the invariant measure simulated via SGLD that
approximates $\varpi_{\Psi_{\btheta}}(\bx)$. $\varpi_{\Psi_{\btheta}}(\bx)$ and $\varpi_{\widetilde{\Psi}_{\btheta}}(\bx)$ are two invariant measures that follow $\varpi_{\Psi_{\btheta}}(\bx)\propto\frac{\pi(\bx)}{\Psi^{\zeta}_{\btheta}(U(\bx))}$ and
 $\varpi_{\widetilde{\Psi}_{\btheta}}(\bx) \propto \frac{\pi(\bx)}{\widetilde{\Psi}^{\zeta}_{\btheta}(U(\bx))}$; $\Psi_{\btheta}(u)$ and $\widetilde{\Psi}_{\btheta}(u)$ are piecewise continuous and constant functions, respectively
\begin{equation}
\begin{split}
\label{new_design_appendix}
\Psi_{\btheta}(u)&= \sum_{k=1}^m \left(\theta(k-1)e^{(\log\theta(k)-\log\theta(k-1)) \frac{u-u_{k-1}}{\Delta u}}\right) 1_{u_{k-1} < u \leq u_k};\ \ \ \widetilde{\Psi}_{\btheta}(u)=\sum_{k=1}^m \theta(k) 1_{u_{k-1} < u \leq u_{k}}.\\
\end{split}
\end{equation}

(i) For the first term $\text{I}_1$, we have
\begin{equation}
\begin{split}
\label{i_1}
    \int_{\MX} H_i(\btheta,\bx) 
     \varpi_{\widetilde{\Psi}_\btheta}(\bx) d\bx&=\frac{1}{\widetilde Z_{\zeta+1,\btheta}} \int_{\MX} {\theta}(J(\bx))\left(1_{i= J(\bx)}-{\theta}(i)\right) \frac{\pi(\bx)}{\theta^{\zeta}(J(\bx))} d\bx\\
     &=\frac{1}{\widetilde Z_{\zeta+1,\btheta}} \sum_{k=1}^m\int_{\MX_k} \left(1_{i= k}-{\theta}(i)\right) \frac{\pi(\bx)}{\theta^{\zeta-1}(k)} d\bx\\
    &=\frac{1}{\widetilde Z_{\zeta+1,\btheta}}\left[\sum_{k=1}^m \int_{\MX_k} 
     \frac{\pi(\bx)}{\theta^{\zeta-1}(k)} 1_{k=i} d\bx -\theta(i)\sum_{k=1}^m\int_{\MX_k} \frac{\pi(\bx)}{\theta^{\zeta-1}(k)}d\bx \right] \\
    &=\frac{1}{\widetilde Z_{\zeta+1,\btheta}}\left[
     \frac{\int_{\MX_i} \pi(\bx)d\bx}{\theta^{\zeta-1}(i)} -\theta(i) \widetilde Z_{\zeta,\btheta} \right], \\
\end{split}
\end{equation}
where $\widetilde Z_{\zeta+1,\btheta}=\sum_{k=1}^m  \frac{\int_{\MX_k} \pi(\bx)d\bx}{\theta^{\zeta}(k)}$ denotes the normalizing constant of $\varpi_{\widetilde{\Psi}_\btheta}(\bx)$.

The solution $\btheta_{\star}$ that solves $\frac{\int_{\MX_k} \pi(\bx)d\bx}{\theta^{\zeta-1}(k)} -\theta(k) \widetilde Z_{\zeta,\btheta}=0$ for any $k\in\{1,2,\cdots, m\}$ satisfies $\theta_{\star}(k)=\left(\frac{\int_{\MX_k} \pi(\bx)d\bx}{\widetilde Z_{\zeta, \btheta_{\star}}}\right)^{\frac{1}{\zeta}}$.
Combining the definition of $\widetilde Z_{\zeta, \btheta_{\star}}=\sum_{k=1}^m  \frac{\int_{\MX_k} \pi(\bx)d\bx}{\theta_{\star}^{\zeta-1}(k)}$, we have
\begin{equation*}
    \begin{split}
    \widetilde Z_{\zeta, \btheta_{\star}}&
    =\sum_{k=1}^m \frac{\int_{\MX_k} \pi(\bx)d\bx}{\theta_{\star}^{\zeta-1}(k)}
    =\sum_{k=1}^m \frac{\int_{\MX_k} \pi(\bx)d\bx}{\left(\frac{\int_{\MX_k} \pi(\bx)d\bx}{\widetilde Z_{\zeta, \btheta_{\star}}}\right)^{\frac{\zeta-1}{\zeta}}}\\
    &
    =\widetilde Z_{\zeta, \btheta_{\star}}^{\frac{\zeta-1}{\zeta}} \sum_{k=1}^m \frac{\int_{\MX_k} \pi(\bx)d\bx}{\left(\int_{\MX_k} \pi(\bx)d\bx\right)^{\frac{\zeta-1}{\zeta}}}
    =\widetilde Z_{\zeta, \btheta_{\star}}^{\frac{\zeta-1}{\zeta}} \sum_{k=1}^m \left(\int_{\MX_k} \pi(\bx)d\bx\right)^{\frac{1}{\zeta}},\\
    \end{split}
\end{equation*}
which leads to $\widetilde Z_{\zeta, \btheta_{\star}}=\left(\sum_{k=1}^m \left(\int_{\MX_k} \pi(\bx)d\bx\right)^{\frac{1}{\zeta}}\right)^{\zeta}$. In other words, the mean-field system without perturbations yields a unique solution $\theta_{\star}(i)=\frac{\left(\int_{\MX_i} \pi(\bx)d\bx\right)^{\frac{1}{\zeta}}}{\sum_{k=1}^m \left(\int_{\MX_k} \pi(\bx)d\bx\right)^{\frac{1}{\zeta}}}$ for any $i\in\{1,2,\cdots, m\}$.

(ii) For the second term $\text{I}_2$, we have 
\begin{equation} \label{biasI2}
\int_{\MX} H_i(\btheta,\bx) (-\varpi_{\widetilde{\Psi}_{\btheta}}(\bx)+\varpi_{\Psi_{\btheta}}(\bx)) d\bx= \mathcal{O}\left(\frac{1}{m}\right),
\end{equation}
where the result follows from the boundedness of $H(\btheta,\bx)$ in (\ref{ass2a}) and Lemma B4 \citep{CSGLD}.

(iii) For the last term $\text{I}_3$, following Theorem 6 of \citet{Sato2014ApproximationAO}, we have for any fixed $\btheta$,
\begin{equation}\label{iiii_2}
    \int_{\MX} H_i(\btheta,\bx) \left(-\varpi_{\Psi_{\btheta}}(\bx)+\varpi_{\btheta}(\bx)\right) d\bx=\mathcal{O}(\epsilon).
\end{equation}

Plugging Eq.(\ref{i_1}), Eq.(\ref{biasI2}) and  Eq.(\ref{iiii_2}) into Eq.(\ref{iiii}), we have
\begin{equation}
\begin{split}
\label{h_i_theta}
     h_i(\btheta)&={\widetilde Z_{\zeta+1,\btheta}}^{-1} \left[\varepsilon\tilde\beta_i(\btheta)+ \frac{\int_{\MX_i} \pi(\bx)d\bx}{\theta^{\zeta-1}(i)} -\theta(i) \widetilde Z_{\zeta,\btheta}\right]\\
     &={\widetilde Z_{\zeta+1,\btheta}}^{-1} \frac{\widetilde Z_{\zeta, \btheta_{\star}}}{\theta^{\zeta-1}(i)}\left[\varepsilon\tilde\beta_i(\btheta)\frac{\theta^{\zeta-1}(i)}{{\widetilde Z_{\zeta, \btheta_{\star}}}}+ \frac{\int_{\MX_i} \pi(\bx)d\bx}{\widetilde Z_{\zeta, \btheta_{\star}}}-\theta^{\zeta}(i) \frac{\widetilde Z_{\zeta,\btheta}}{\widetilde Z_{\zeta, \btheta_{\star}}} \right]\\
     &={\widetilde Z_{\zeta+1,\btheta}}^{-1} \frac{\widetilde Z_{\zeta, \btheta_{\star}}}{\theta^{\zeta-1}(i)}\left[\varepsilon\tilde\beta_i(\btheta)\frac{\theta^{\zeta-1}(i)}{{\widetilde Z_{\zeta, \btheta_{\star}}}} +\theta_{\star}^{\zeta}(i)-\left(\theta(i) C_{\btheta}\right)^{\zeta} \right],
\end{split}
\end{equation}
where $\tilde\beta_i(\btheta)$ is a bounded term such that ${\widetilde Z_{\zeta+1,\btheta}}^{-1}\varepsilon\tilde\beta_i(\btheta)=\mathcal{O}\left(\Var(\xi_n(\bx))+\epsilon+\frac{1}{m}\right)$,  $\varepsilon=\mathcal{O}\left(\sup_{\bx}\Var(\xi_n(\bx))+\epsilon+\frac{1}{m}\right)$ and $C_{\btheta}=\left(\frac{\widetilde Z_{\zeta,\btheta}}{\widetilde Z_{\zeta, \btheta_{\star}}}\right)^{\frac{1}{\zeta}}$. By the definition of $\widetilde Z_{\zeta,\btheta}=\sum_{k=1}^m  \frac{\int_{\MX_k} \pi(\bx)d\bx}{\theta^{\zeta-1}(k)}$, when $\zeta=1$, $C_{\btheta}\equiv 1$ for any $\btheta\in \bTheta$, which suggests that the stability condition doesn't rely on the initialization of $\btheta$; however, when $\zeta\neq 1$,  $C_{\btheta}\neq 1$ when $\btheta\neq \btheta_{\star}$, we see that $h_i(\btheta)\propto \theta_{\star}(i)^{\zeta}-\left(\theta(i) C_{\btheta}\right)^{\zeta}+\text{perturbations}$ is a non-linear mean-field system and requires a proper initialization of $\btheta\in\widetilde \bTheta$. 

For any $\btheta\in\widetilde\bTheta\subset \bTheta$ being close enough to $\btheta_{\star}$, there exists a Lipschitz constant $L_{\widetilde \btheta}=\sup_{i\leq m, \btheta\in\widetilde\bTheta} \frac{|C_{\btheta_{\star}}-C_{\btheta}|}{|\theta_{\star}(i)-\theta(i)|}<\infty$. By $C_{\btheta_{\star}}=1$, $\theta(i)\leq 1$, and mean value theorem for some $\widetilde \theta(i)\in[\theta(i), \theta_{\star}(i)]$, we have
\begin{align}\label{mean-value-decompose}
    |\theta_{\star}^{\zeta}(i)-\left(\theta(i) C_{\btheta}\right)^{\zeta}|&= \zeta (\widetilde\theta(i)C_{\widetilde\btheta})^{\zeta-1}|\theta_{\star}(i)-\theta(i) C_{\btheta}|\notag\\
    &=\zeta (\widetilde\theta(i)C_{\widetilde\btheta})^{\zeta-1}|\theta_{\star}(i)-\theta(i)+\theta(i)C_{\btheta_{\star}}-\theta(i) C_{\btheta}|\notag\\
    &\leq \zeta (\widetilde\theta(i)C_{\widetilde\btheta})^{\zeta-1}|\theta_{\star}(i)-\theta(i)|+\theta(i) |C_{\btheta_{\star}}-C_{\btheta}|\notag\\
    &\leq \zeta (\widetilde\theta(i)C_{\widetilde\btheta})^{\zeta-1} (1+ L_{\widetilde \btheta})|\theta_{\star}(i)-\theta(i)|,
\end{align}

Combining Eq.(\ref{h_i_theta}) and Eq.(\ref{mean-value-decompose}), we have
\begin{equation}
\begin{split}
\label{h_i_theta_v2}
     h_i(\btheta)&={\widetilde Z_{\zeta+1,\btheta}}^{-1} \frac{\widetilde Z_{\zeta, \btheta_{\star}}}{\theta^{\zeta-1}(i)}\left[\varepsilon\tilde\beta_i(\btheta)\frac{\theta^{\zeta-1}(i)}{{\widetilde Z_{\zeta, \btheta_{\star}}}} +\theta_{\star}^{\zeta}(i)-\left(\theta(i) C_{\btheta}\right)^{\zeta} \right]\\
     &=\widehat Z_{\zeta, \theta(i)}^{-1} \left[\varepsilon\beta_i(\btheta) +\theta_{\star}(i)-\theta(i)  \right],\\
\end{split}
\end{equation}
where $\widehat Z_{\zeta,\theta(i)}^{-1}=\frac{{\widetilde Z_{\zeta+1,\btheta}}^{-1}\widetilde Z_{\zeta, \btheta_{\star}}}{\zeta(\widetilde \theta(i)C_{\widetilde\btheta})^{\zeta-1}(1+L_{\widetilde\btheta})\theta^{\zeta-1}(i)}$; $\beta_i(\btheta)$ is some bounded term such that $\beta_i(\btheta)\leq \frac{\tilde\beta_i(\btheta)\theta^{\zeta-1}(i)}{\zeta (\widetilde\theta(i)C_{\widetilde\btheta})^{\zeta-1} (1+ L_{\widetilde \btheta}) {\widetilde Z_{\zeta, \btheta_{\star}}}}$; $C_{\widetilde\btheta}=\left(\frac{\widetilde Z_{\zeta,\widetilde\btheta}}{\widetilde Z_{\zeta, \btheta_{\star}}}\right)^{\frac{1}{\zeta}}$; $L_{\widetilde \btheta}=\sup_{i\leq m, \btheta\in\widetilde\bTheta} \frac{|C_{\btheta_{\star}}-C_{\btheta}|}{|\theta_{\star}(i)-\theta(i)|}<\infty$.

Next, we apply the perturbation theory to solve the ODE system with small disturbances \citep{Perturbation2} and obtain the equilibrium $\widehat\btheta_{\star}$,

where $\varepsilon\bbeta(\widehat\btheta_{\star})+\btheta_{\star}-\widehat \btheta_{\star}=0$, to the mean-field equation $h_i(\btheta)$ such that
\begin{equation}
\begin{split}
    h_i(\btheta)&=\widehat Z_{\zeta, \theta(i)}^{-1} \left[\varepsilon\beta_i(\theta)+\theta_{\star}(i)-\theta(i)\right]\\
    &=\widehat Z_{\zeta, \theta(i)}^{-1} \left[\varepsilon\beta_i(\theta)-\varepsilon\beta_i(\widehat\theta_{\star})+\varepsilon\beta_i(\widehat\theta_{\star})+\theta_{\star}(i)-\theta(i)\right]\\
     &=\widehat Z_{\zeta, \theta(i)}^{-1} \left[\mathcal{O}(\varepsilon)(\theta(i)-\widehat\theta_{\star}(i))+\widehat \theta_{\star}(i)-\theta(i)\right]\\
    &=\widehat Z_{\zeta, \theta(i)}^{-1} \big(1-\mathcal{O}(\varepsilon)\big)\left(\widehat\theta_{\star}(i)-\theta(i)\right),\\
\end{split}
\end{equation}
where a smoothness condition clearly holds for the $\beta(\cdot)$ function.  Given a positive definite Lyapunov function $\mathbb{V}(\btheta)=\frac{1}{2}\| \widehat\btheta_{\star}-\btheta\|^2$, the mean-field system $h(\btheta)=\widehat Z_{\zeta,\theta(i)}^{-1} (\varepsilon\bbeta(\btheta)+\btheta_{\star}-\btheta)=\widehat Z_{\zeta,\theta(i)}^{-1} (1-\mathcal{O}(\varepsilon)) (\widehat\btheta_{\star}-\btheta)$ for $i\in\{1,2,\cdots, m\}$ enjoys the following property
\begin{equation*}
\begin{split}
     \langle h(\btheta), \nabla\mathbb{V}(\btheta)\rangle&=\langle h(\btheta), \btheta - \widehat\btheta_{\star}\rangle \\
     &\leq  -\min_{i}\widehat Z_{\zeta,\theta(i)}^{-1} \big(1-\mathcal{O}(\varepsilon)\big)\|\btheta -\widehat \btheta_{\star}\|^2\\
    &\leq -\phi\|\btheta - \widehat\btheta_{\star}\|^2,
\end{split}
\end{equation*}
where $\phi=\inf_{\btheta}\min_i\widehat Z_{\zeta,\theta(i)}^{-1}\big(1-\mathcal{O}(\varepsilon)\big) >0$ given
the compactness assumption \ref{ass2a} and a small enough $\varepsilon=\mathcal{O}\left(\sup_{\bx}\Var(\xi_n(\bx))+\epsilon+\frac{1}{m}\right)$. \qed
\end{proof}

\begin{remark}
The newly proposed random-field function Eq.(\ref{random_field_H}) may sacrifice the global stability by including an approximately linear mean-field system Eq.(\ref{h_i_theta_v2}) instead of a linear stable system (see formula (15) in \citet{CSGLD}). The advantage, however, is that such a mechanism facilitates the estimation of $\btheta_{\star}$. We emphasize that the original energy probability in each partition  $\big\{\int_{\MX_k} \pi(\bx)d\bx\big\}_{k=1}^m$ \citep{CSGLD} may be very difficult to estimate for big data problems. By contrast, the estimation of $\big\{\big(\int_{\MX_k} \pi(\bx)d\bx\big)^{\frac{1}{\zeta}}\big\}_{k=1}^m$ becomes much easier given a proper $\zeta>0$.
\end{remark}

\paragraph{Technical lemmas}
\begin{lemma}\label{bias_in_SA}
The stochastic energy estimator $\widetilde U(\bx)$ leads to a controllable bias in the random-field function. 
\label{bias_sa}
\begin{equation*}\label{eq_bias_in_SA}
 |\E[\widetilde H_i(\btheta,\bx)]- H_i(\btheta,\bx)|=\mathcal{O}\left(\Var(\xi_n(\bx))\right),
\end{equation*}
where the expectation $\E[\cdot]$ is taken with respect to the random noise in the stochastic energy estimator of $\widetilde U(\cdot)$. 
\end{lemma}

\begin{proof}
Denote the noise in the stochastic energy estimator by $\xi(\bx)$, such that $\widetilde U(\cdot)=U(\cdot) + \xi(\cdot)$. Recall that $\widetilde H_i(\btheta,\bx)={\theta}(J_{\widetilde U}(\bx))\left(1_{i= J_{\widetilde U}(\bx)}-{\theta}(i)\right)$ and $J_{\widetilde U}(\bx)\in\{1,2,\cdots, m\}$ satisfies $u_{J_{\widetilde U}(\bx)-1}< \frac{N}{n}\widetilde U(\bx)\leq u_{J_{\widetilde U}(\bx)}$ for a set of energy partitions $\{u_i\}_{i=0}^{m}$. We can interpret $\widetilde H_i(\btheta,\bx)$ as a non-linear transformation $\Phi$ that maps $\widetilde U(\bx)$ to $(0, 1)$. Similarly, $ H_i(\btheta,\bx)$ maps $U(\bx)$ to $(0, 1)$. In what follows, the bias of random-field function is upper bounded as follows
\begin{equation*}
\begin{split}
     |\E[\widetilde H_i(\btheta,\bx)]- H_i(\btheta,\bx)|&=\left|\int \Phi(U(\bx)+\xi(\bx))-\Phi(U(\bx))d\mu(\xi(\bx))\right|\\
     &=\left|\int \xi(\bx) \Phi'(U(\bx))+\frac{\xi(\bx)^2}{2} \Phi''(u) d\mu(\xi(\bx))\right|\\
     &= \mathcal{O}\left(\Var(\xi_n(\bx))\right), \\
\end{split}
\end{equation*}
where the second equality follows from Taylor expansion for some energy $u$ and the third equality follows because the stochastic energy estimator is unbiased; $\Phi'(U(\bx))=\mathcal{O}( \frac{\theta(J(\bx))-\theta(J(\bx)-1)}{\Delta u})$ is clearly bounded due to the definition of $\btheta$; a similar conclusion also applies to $\Phi''(\cdot)$.

\end{proof}

\subsection{Convergence of the self-adapting parameters}

The following is a restatement of Lemma 3.2 of \citet{Maxim17}, which holds for any $\btheta$ in the compact space $\bTheta$.
\begin{lemma}[Uniform $L^2$ bounds]
\label{lemma:1}
Assume Assumptions \ref{ass2a}, \ref{ass3} and \ref{ass4} hold. We have a bounded second moment
$\sup_{k\geq 1} \E[\|\bx_{k}\|^2] < \infty$ given a small enough learning rate.
\end{lemma}

The following lemma justifies the regularity properties of Poisson's equation, which is crucial in controlling the perturbations through the stochastic approximation process. The first version was proposed in Lemma B2 of \citet{CSGLD}. Now we give a more detailed proof by utilizing a Lyapunov function $V(\bx)=1+\bx^2$ and Lemma \ref{lemma:1}.
\begin{lemma}[Solution of Poisson's equation]
\label{lyapunov_ori}
Assume that Assumptions  \ref{ass2a}-\ref{ass4}  hold. 
There is a solution $\mu_{\btheta}(\cdot)$ on $\MX$ to the Poisson's equation 
\begin{equation}
    \label{poisson_eqn_ori}
    \mu_{\btheta}(\bm{x})-\mathrm{\Pi}_{\bm{\theta}}\mu_{\bm{\theta}}(\bm{x})=\widetilde H(\bm{\theta}, \bm{x})-h(\bm{\theta}).
\end{equation}
Furthermore, there exists a constant $C$ such that for all $\bm{\theta}, \bm{\theta}'\in \bm{\bTheta}$
\begin{equation}
\begin{split}
\label{poisson_reg}
\E[\|\mathrm{\Pi}_{\bm{\theta}}\mu_{\btheta}(\bx)\|]&\leq C,\\
\E[\|\mathrm{\Pi}_{\bm{\theta}}\mu_{\bm{\theta}}(\bx)-\mathrm{\Pi}_{\bm{\theta}'}\mu_{\bm{\theta'}}(\bx)\|]&\leq C\|\bm{\theta}-\bm{\theta}'\|.\\
\end{split}
\end{equation}
\end{lemma}

\begin{proof} The existence and the regularity property of Poisson's equation can be used to control the perturbations. The key of the proof lies in verifying drift conditions proposed in Section 6 of \citet{AndrieuMP2005}.

\textbf{(DRI)} By the smoothness assumption \ref{ass2}, we have that $U(\bx)$ is continuously differentiable almost everywhere. By the dissipative assumption \ref{ass3} and Theorem 2.1 \citep{Roberts_Tweedie_Bernoulli}, we can show that the discrete dynamics system is irreducible and aperiodic. Now consider a Lyapunov function $V=1+\|\bx\|^2$ and any compact subset $\mathcal{\bm{K}}\subset \bTheta$, the drift conditions are verified as follows:

\textbf{(DRI1)} Given small enough learning rates $\{\epsilon_k\}_{k\geq 1}$, the smoothness assumption \ref{ass2}, and the dissipative assumption \ref{ass3}, applying Corollary 7.5 \citep{mattingly02} yields the minorization condition for the CSGLD algorithm, i.e. there exists $\eta>0$, a measure $\nu$, and a set $\mathcal{C}$ such that $\nu(\mathcal{C})=1$. Moreover, we have
    $$P_{\btheta\in \mathcal{\bm{K}}}(x, A)\geq  \eta \nu(A)\ \ \ \ \ \forall A\in \MX, \bx\in \mathcal{C}. \eqno{(\text{I})}$$
    
where $P_{\btheta}(\bx, \by):=\frac{1}{2\sqrt{(4\pi\epsilon)^{d/2}}}\E\big[e^{-\frac{\|\by-\bx+\epsilon \nabla_{\bx} \widetilde{L}(\bx, \btheta)\|^2}{4\epsilon}}|\bx\big]$ denotes the transition kernel based on CSGLD with the parameter $\btheta\in\mathcal{\bm{K}}$ and a learning rate $\epsilon$, in addition, the expectation is taken over the adaptive gradient $\nabla_{\bx} \widetilde{L}(\bx, \btheta)$ in Eq.(\ref{adaptive_grad}). Using Assumption \ref{ass2a}-\ref{ass4}, we can prove the uniform L2 upper bound by following Lemma 3.2 \citep{Maxim17}. Further, by Theorem 7.2 \citep{mattingly02}, there exist $\tilde\alpha\in (0,1)$ and $\tilde \beta\geq 0$ such that
    $$P_{\btheta\in\mathcal{\MK}}V(\bx)\leq \tilde\alpha V(\bx)+\tilde\beta. \eqno{(\text{II})}$$

Consider a Lyapunov function $V=1+\|\bx\|^2$ and a constant $\kappa=\tilde\alpha+\tilde \beta$, it yields that
$$P_{\btheta\in\mathcal{\bm{K}}}V(\bx)\leq \kappa V(\bx). \eqno{(\text{III})}$$
Now we have verified the first condition \text{(DRI1)} by checking conditions (\text{I}),(\text{II}), and (\text{III}), 

\textbf{(DRI2)} In what follows, we check the boundedness and Lipshitz conditions on the random-field function $\widetilde H(\btheta,\bx)$, where each subcomponent is defiend as $\widetilde H_i(\btheta,\bx)={\theta}(J_{\widetilde U}(\bx))\left(1_{i=J_{\widetilde U}(\bx)}-{\theta}(i)\right)$. Recall that $V=1+\|\bx\|^2$, the compactness assumption \ref{ass2a} directly leads to
$$\sup_{\btheta\in\mathcal{\bm{K}}\subset [0, 1]^m}\| H(\btheta, \bx)\|\leq m V(\bx). \eqno{(\text{IV})}$$
For any $\btheta_1, \btheta_2\in \mathcal{\bm{K}}$ and a fixed $\bx\in\MX$, it suffices for us to solely verify the $i$-th index, which is the index that maximizes $|\theta_1(i)-\theta_2(i)|$, then
\begin{equation*}
\begin{split}
\small
    |\widetilde H_i(\btheta_1,\bx)- \widetilde H_i(\btheta_2,\bx)|&={\theta_1} (J_{\widetilde U}(\bx))\left(1_{i=J_{\widetilde U}(\bx)}-{\theta_1}(i)\right)-{\theta_2} (J_{\widetilde U}(\bx))\left(1_{i=J_{\widetilde U}(\bx)}-{\theta_2}(i)\right)\\
    &\leq |{\theta_1} (J_{\widetilde U}(\bx))-{\theta_2} (J_{\widetilde U}(\bx))|+|{\theta_1} (J_{\widetilde U}(\bx)){\theta_1}(i)-{\theta_2} (J_{\widetilde U}(\bx)){\theta_2}(i)|\\
    &\leq \max_{j}\Big(|{\theta_1} (j)-{\theta_2} (j)|+{\theta_1} (j)|{\theta_1}(i)-{\theta_2}(i)|+|{\theta_1} (j)-{\theta_2} (j)|\theta_2(i)\Big)\\
    &\leq 3|\theta_1(i)-\theta_2(i)|,\\
\end{split}
\end{equation*}
where the last inequality holds since $\theta(i)\in(0, 1]$ for any $i\leq m$.

\textbf{(DRI3)} We proceed to verify the smoothness of the transitional kernel $P_{\btheta}(\bx, \by)$ with respect to $\btheta$. For any $\btheta_1, \btheta_2\in \mathcal{\bm{K}}$ and fixed $\bx$ and $\by$, we have 
\begin{equation*}
\begin{split}
    &\quad|P_{\btheta_1}(\bx, \by)-P_{\btheta_2}(\bx, \by)|\\
    &=\frac{1}{2\sqrt{(4\pi\epsilon)^{d/2}}}\E\big[e^{-\frac{\|\by-\bx+\epsilon \nabla_{\bx} \widetilde{L}(\bx, \btheta_1)\|^2}{4\epsilon}}|\bx\big]-\frac{1}{2\sqrt{(4\pi\epsilon)^{d/2}}}\E\big[e^{-\frac{\|\by-\bx+\epsilon \nabla_{\bx} \widetilde{L}(\bx, \btheta_2)\|^2}{4\epsilon}}|\bx\big]\\
    &\lesssim |\|\by-\bx+\epsilon \nabla_{\bx} \widetilde{L}(\bx, \btheta_1)\|^2-\|\by-\bx+\epsilon \nabla_{\bx} \widetilde{L}(\bx, \btheta_2)\|^2|\\
    &\lesssim \|\nabla_{\bx} \widetilde{L}(\bx, \btheta_1)- \nabla_{\bx} \widetilde{L}(\bx, \btheta_2)\|\\
    &\lesssim \| \btheta_1-\btheta_2\|,\\
\end{split}
\end{equation*}
where the first inequality (up to a finite constant) follows by $\|e^{\bx}-e^{\by}\|\lesssim \|\bx-\by\|$ for any $\bx$, $\by$ in a compact space; the last inequality follows by the definition of the adaptive gradient in Eq.(\ref{adaptive_grad}) and $\|\log(\bx)-\log(\by)\|\lesssim \|\bx-\by\|$ by the compactness assumption \ref{ass2a}.

For $f:\MX\rightarrow\mathbb{R}^d$, define the norm $\|f\|_V=\sup_{\bx\in\MX} \frac{|f(\bx)|}{V(\bx)}$. Following the same technique proposed in \citet{Liang07} (page 319), we can verify the last drift condition
$$\|P_{\btheta_1}f-P_{\btheta_2}f\|_V\leq C\|f\|_V \|\btheta_1-\btheta_2\|, \ \ \forall f\in \mathcal{L}_V:=\{f: \MX\rightarrow \mathbb{R}^d, \|f\|_V<\infty\}. \eqno{(\text{VI})}$$

Having conditions (\text{I}), (\text{II}), $\cdots$ and (\text{VI}) verified, we are now able to prove the drift conditions proposed in Section 6 of \citet{AndrieuMP2005}.\qed
\end{proof}

Before we present the  $L^2$
convergence of $\btheta_k$,  we make some extra assumptions on the step size.

\begin{assump}[Learning rate and step size]
\label{ass_step_size}
The learning rate $\{\epsilon_k\}_{k \in \mathrm{N}}$ is a positive non-increasing sequence of real numbers satisfying the conditions 
\[
\lim_k \epsilon_k=0, \quad \sum_{k=1}^{\infty} \epsilon_k=\infty.
\]
The step size $\{\omega_{k}\}_{k\in \mathrm{N}}$ is a positive non-increasing  sequence of real numbers such that
\begin{equation} \label{a1}
\lim_{k \to \infty} \omega_k=0, \quad
\sum_{k=1}^{\infty} \omega_{k}=+\infty, \quad  \sum_{k=1}^{\infty} \omega_{k}^2<+\infty.
\end{equation}
A practical strategy is to set $\omega_{k}:=\mathcal{O}(k^{-\alpha})$ to satisfy the above conditions for any $\alpha\in (0.5, 1]$. 

 \end{assump}

The following is an application of Theorem 24 (page 246) \citep{Albert90} given stability conditions (Lemma \ref{convex_appendix}).
\begin{lemma}[$L^2$ convergence rate, restatement of Lemma \ref{latent_convergence_main}]
\label{latent_convergence_appendix}
Assume Assumptions $\ref{ass2a}$-$\ref{ass_step_size}$ hold. For any $\btheta_{0}\in\widetilde\bTheta\subset\bTheta$, a large $m$, small learning rates $\{\epsilon_k\}_{k=1}^{\infty}$, and step sizes $\{\omega_k\}_{k=1}^{\infty}$, 
$\{\btheta_k\}_{k=0}^{\infty}$ converges to $\widehat\btheta_{\star}$,
where $\widehat\btheta_{\star}=\btheta_{\star}+\mathcal{O}\left(\sup_{\bx}\Var(\xi_n(\bx))+\sup_{k\geq k_0}\epsilon_k+\frac{1}{m}\right)$ for some $k_0$, such that
\begin{equation*}
    \E\left[\|\bm{\theta}_{k}-\widehat\btheta_{\star}\|^2\right]= \mathcal{O}\left(\omega_{k}\right).
\end{equation*}
\end{lemma}

The theoretical novelty is that we treat the biased $\widehat\btheta_{\star}$ as the equilibrium of the continuous system instead of analyzing how far we are away from $\btheta_{\star}$ in all aspects as in Theorem 1 \citep{CSGLD}. This enables us to directly apply Theorem 24 (page 246). Nevertheless, it can be interpreted as a special case of Theorem 1 \citep{CSGLD} except that there are no perturbation terms and the equilibrium is $\widehat\btheta_{\star}$ instead of $\btheta_{\star}$.

\section{Gaussian approximation}
\label{Gaussian_approx}

\subsection{Preliminary: sufficient conditions for weak convergence}

To formally prove the asymptotic normality of the stochastic approximation process $\omega_k^{-1/2}(\btheta_k-\widehat\btheta_{\star})$, we first lay out a preliminary result (Theorem 1 of \citet{Pelletier98}) that provides sufficient conditions to guarantee the weak convergence.

\begin{lemma}[Sufficient Conditions]
\label{sufficiency}
Consider a stochastic algorithm as follows
\begin{equation*}
    \btheta_{k+1}=\bm{\theta}_{k}+\omega_{k+1}h(\bm{\theta}_{k}) +\omega_{k+1} \bm{\widetilde\nu}_{k+1}+\omega_{k+1}\bm{e}_{k+1},
\end{equation*}
where $\bm{\widetilde\nu}_{k+1}$ denotes a perturbation and $\bm{e}_{k+1}$ is a random noise. Given three conditions (\textbf{C1}), (\textbf{C2}), and (\textbf{C3}) defined below, we have the desired weak convergence result
\begin{equation}
    \omega^{-\frac{1}{2}}(\btheta_k -\widehat\btheta_{\star})\Rightarrow\mathcal{N}(0, \bSigma),
\end{equation}
where $\bSigma=\int_0^{\infty} e^{t h_{\btheta_{\star}}}\circ \bR\circ  e^{th^{\top}_{\btheta_{\star}}}dt$, $\bR$ denotes the limiting covariance of the martingale $\lim_{k\rightarrow\infty}\E[\bm{e_{k+1}}\bm{e_{k+1}}^{\top}|\mathcal{F}_k]$ and $\mathcal{F}_k$ is the $\sigma$-algebra of the events up to iteration $k$, $h_{\btheta_{\star}}=h_{\btheta}(\widehat\btheta_{\star})+\widehat\xi\bI$, $\widehat\xi=\lim_{k\rightarrow \infty}\frac{\omega_k^{0.5}-\omega_{k+1}^{0.5}}{\omega_k^{1.5}}$. \footnote[2]{For example, $\widehat\xi=0$ if $\omega_k=\mathcal{O}(k^{-\alpha})$, where $\alpha\in(0.5, 1]$ and $\widehat\xi=\frac{k_0}{2}$ if  $\omega_k=\frac{k_0}{k}$.}

\textbf{(C1)} There exists an equilibrium point $\widehat\btheta_{\star}$ and a stable matrix $h_{\btheta_{\star}}:=h_{\btheta}(\widehat\btheta_{\star})\in\mathbb{R}^{m\times m}$ such that for any $\btheta\in \{\btheta: \|\btheta-\widehat\btheta_{\star}\|\leq \widetilde M\}$ for some $\widetilde M>0$, the mean-field function $h:\mathbb{R}^m\rightarrow \mathbb{R}^m$ satisfies
\begin{equation*}
\begin{split}
    h(\widehat\btheta_{\star})&=0\\
    \|h(\btheta)-h_{\btheta_{\star}} (\btheta-\widehat\btheta_{\star})\|&\lesssim \|\btheta-\widehat\btheta_{\star}\|^2,
\end{split}
\end{equation*}

\textbf{(C2)} The step size $\omega_k$ decays with an order $\alpha\in (0, 1]$ such that $\omega_k=\mathcal{O}(k^{-\alpha})$.

\textbf{(C3)} Assumptions on the disturbances . There exists constants $\widetilde M>0$ and $\widetilde \alpha>2$ such that
$$ \E\left[\bm{e}_{k+1}|\mathcal{F}_k\right]\bm{1}_{\{\|\btheta-\widehat\btheta_{\star}\|\leq \widetilde M\}}=0,  \eqno{(\text{I}_1)}$$
$$  \sup_k\E\left[\|\bm{e}_{k+1}\|^{\widetilde \alpha}|\mathcal{F}_k\right]\bm{1}_{\{\|\btheta-\widehat\btheta_{\star}\|\leq \widetilde M\}}< \infty,  \eqno{(\text{I}_2)}$$
$$ \E\left[\omega_k^{-1}\|\bm{\widetilde \nu}_{k+1}\|^2\right]\bm{1}_{\{\|\btheta-\widehat\btheta_{\star}\|\leq \widetilde M\}}\rightarrow 0,  \eqno{(\text{II})}$$
$$ \E\left[\bm{e}_{k+1} \bm{e}_{k+1}^{\top}|\mathcal{F}_k\right]\bm{1}_{\{\|\btheta-\widehat\btheta_{\star}\|\leq \widetilde M\}}\rightarrow \bR. \eqno{(\text{III})}$$

\end{lemma}

\begin{remark}
By the definition of the mean-field function $h(\btheta)$ in Eq.(\ref{h_i_theta}), it is easy to verify the condition C1. Moreover, Assumption \ref{ass_step_size} also fulfills the  condition C2. Then, the proof hinges on the verification of the condition C3. 
\end{remark}

\subsection{Preliminary: convergence of the covariance estimators}

In particular, to verify the condition $\E\left[\bm{e}_{k+1} \bm{e}_{k+1}^{\top}|\mathcal{F}_k\right]\bm{1}_{\{\|\btheta-\widehat\btheta_{\star}\|\leq \widetilde M\}}\rightarrow \bR$, , we study the convergence of the 
empirical sample mean $\E[f(\bx_k)]$ for a test function $f$ to the posterior expectation $\bar{f}=\int_{\MX}f(\bx)\varpi_{\widehat\btheta_{\star}}(\bx)(d\bx)$. Poisson's equation is often used to characterize the fluctuation 
between $f(\bx)$ and $\bar f$: 
\begin{equation}
    \mathcal{L}g(\bx)=f(\bx)-\bar f,
\end{equation}
where $\mathcal{L}$ refers to an infinitesimal generator and $g(\bx)$ denotes the solution of the Poisson's equation. Similar to the proof of Lemma \ref{lyapunov_ori}, the existence of the solution of the Poisson's equation has been established in \citep{mattingly02, VollmerZW2016}. Moreover, the perturbations of $\E[f(\bx_k)]-\bar f$ are properly bounded given regularity properties for $g(\bx)$, where the 0-th, 1st, and 2nd order of the regularity properties has been established in \citet{Mackey18}.

The following result helps us to identify the convergence of the covariance estimators, which is adapted from Theorem 5 \citep{Chen15} with decreasing learning rates $\{\epsilon_k\}_{k\geq 1}$. The gradient biases from Theorem 2 \citep{Chen15} are also included to handle the adaptive biases.

\begin{lemma}[Convergence of the Covariance Estimators]
\label{covariance_estimator}
Suppose Assumptions \ref{ass2a}-\ref{ass_step_size} hold. For any $\btheta_{0}\in\widetilde\bTheta\subset\bTheta$, a large $m$, small learning rates $\{\epsilon_k\}_{k=1}^{\infty}$, step sizes $\{\omega_k\}_{k=1}^{\infty}$ and any bounded function $f$, we have 
\begin{equation*}
\begin{split}
    \left|\E\left[f(\bx_k)\right]-\int_{\MX}f(\bx)\varpi_{\widehat\btheta_{\star}}(\bx)d\bx\right|&\rightarrow 0, \\
\end{split}
\end{equation*}
where $\varpi_{\widehat\btheta_{\star}}(\bx)$ is the invariant measure simulated via SGLD that approximates $\varpi_{\widetilde{\Psi}_{\btheta_{\star}}}(\bx) \propto  
\frac{\pi(\bx)}{\theta_{\star}^{\zeta}(J(\bx))}$.
\end{lemma}

\begin{proof} We study the single-chain CSGLD and reformulate the adaptive algorithm as follows:
\begin{equation*}
\begin{split}
    \bx_{k+1}&=\bx_k- \epsilon_k\nabla_{\bx} \widetilde{L}(\bx_k, \btheta_k)+\mathcal{N}({0, 2\epsilon_k \tau\bm{I}})\\
    &=\bx_k- \epsilon_k\left(\nabla_{\bx} 
    \widetilde{L}(\bx_k, \widehat\btheta_{\star})+{\Upsilon}(\bx_k, \btheta_k)\right)+\mathcal{N}({0, 2\epsilon_k \tau\bm{I}}),
\end{split}
\end{equation*}
where    
$\nabla_{\bx} \widetilde{L}(\bx,\btheta)= \frac{N}{n} \left[1+  \frac{\zeta\tau}{\Delta u}  \left(\log \theta({J}(\bx))-\log\theta(({J}(\bx)-1)\vee 1) \right) \right]  \nabla_{\bx} \widetilde U(\bx)$ \footnote[3]{$J(\bx)=\sum_{i=1}^m i 1_{u_{i-1}<U(\bx)\leq u_i}$, where the exact energy function $U(\bx)$ is selected.},  $\nabla_{\bx} \widetilde{L}(\bx,\btheta)$ is 
defined in Section \ref{Alg:app} and the bias term is given by ${\Upsilon}(\bx_k,\btheta_k)=\nabla_{\bx} \widetilde{L}(\bx_k,\btheta_k)-\nabla_{\bx} \widetilde{L}(\bx_k,\widehat\btheta_{\star})$.

Then, by Jensen's inequality and Lemma \ref{latent_convergence_appendix}, we have
\begin{equation}
\label{latent_bias}
\begin{split}
    \|\E[\Upsilon(\bx_k,\btheta_k)]\|&\leq 
    \E[\|\nabla_{\bx} \widetilde{L}(\bx_k, \btheta_k)-\nabla_{\bx} \widetilde{L}(\bx_k, \widehat\btheta_{\star})\|] \\
    &\lesssim  \E[\|\btheta_k-\widehat\btheta_{\star}\|]\leq \sqrt{\E[\|\btheta_k-\widehat\btheta_{\star}\|^2]}\leq \mathcal{O}\left( \sqrt{\omega_{k}}\right).
\end{split}
\end{equation}
Combining Eq.(\ref{latent_bias}) and Theorem 5 \citep{Chen15}, we have
\begin{equation*}
\begin{split}
    \left|\E\left[f(\bx_k)\right]-\int_{\MX}f(\bx)\varpi_{\widehat\btheta_{\star}}(\bx)d\bx\right|&=
    \mathcal{O}\left(\frac{1}{\sum_i^k \epsilon_i }+\frac{\sum_{i=1}^k \omega_i \|\E[\Upsilon(\bx_i,\btheta_i)]\|}{\sum_i^k \omega_i } +\frac{\sum_i^k \epsilon_i^2}{\sum_i^k \epsilon_i }\right)\\
    &\rightarrow 0, \text{\ as}\ k\rightarrow \infty, \\
\end{split}
\end{equation*}
where the last argument directly follows from the conditions on learning rates and step sizes in Assumption \ref{ass_step_size}. \qed
\end{proof}

\subsection{Proof of Theorem \ref{Asymptotic}}
\label{proof_theorem_1}
Recall that the stochastic approximation based on a single process follows from
\begin{equation}
\begin{split}
\label{ga_1}
    &\quad\bm{\theta}_{k+1}\\
    &=\bm{\theta}_{k}+\omega_{k+1} H(\bm{\theta}_{k}, \bx_{k+1})\\
    &=\bm{\theta}_{k}+\omega_{k+1}h(\bm{\theta}_{k}) +\omega_{k+1}\left(\mu_{\btheta_k}(\bx_{k+1})-\Pi_{{\btheta_k}}\mu_{\bm{\theta}_k}(\bx_{k+1})\right)\\
    &=\bm{\theta}_{k}+\omega_{k+1}h(\bm{\theta}_{k})\\
    &\quad+\omega_{k+1}\underbrace{\left(\Pi_{{\btheta_{k+1}}}\mu_{\btheta_{k+1}}(\bx_{k+1})-\Pi_{{\btheta_k}}\mu_{\btheta_k}(\bx_{k+1})+\frac{\omega_{k+2}-\omega_{k+1}}{\omega_{k+1}}\Pi_{{\btheta_{k+1}}}\mu_{\btheta_{k+1}}(\bx_{k+1})\right)}_{\bm{\nu}_{k+1}}\\
    &\quad+\omega_{k+1}\bigg(\underbrace{\frac{1}{\omega_{k+1}}\bigg(\omega_{k+1} \Pi_{{\btheta_{k}}}\mu_{\btheta_{k}}(\bx_{k})- \omega_{k+2} \Pi_{{\btheta_{k+1}}}\mu_{\btheta_{k+1}}(\bx_{k+1})\bigg)}_{\bm{\varsigma}_{k+1}}+\underbrace{\mu_{\btheta_k}(\bx_{k+1})-\Pi_{{\btheta_k}}\mu_{\btheta_k}(\bx_{k})}_{\bm{e}_{k+1}}\bigg) \\
    &=\bm{\theta}_{k}+\omega_{k+1}h(\bm{\theta}_{k}) +\omega_{k+1}\underbrace{\left( \bm{\nu}_{k+1}+\bm{\varsigma}_{k+1}\right)}_{\text{perturbation}}+\omega_{k+1}\underbrace{\bm{e}_{k+1}}_{\text{martingale}},\\
\end{split}
\end{equation}
where the second equality holds from the solution of Poisson's equation in Eq.(\ref{poisson_eqn_ori}).

We denote $\ddot\btheta_k=\btheta_k+\omega_{k+1} \Pi_{{\btheta_{k}}}\mu_{\btheta_{k}}(\bx_{k})$. Adding $\omega_{k+2} \Pi_{{\btheta_{k+1}}}\mu_{\btheta_{k+1}}(\bx_{k+1})$ on both sides of Eq.(\ref{ga_1}), we have
\begin{equation}
\begin{split}
    &\quad\ddot\btheta_{k+1}\\
    &=\ddot\btheta_k+\omega_{k+1}h(\bm{\theta}_{k}) +\omega_{k+1}\left( \bm{\nu}_{k+1}+\bm{e}_{k+1}+\bm{\varsigma}_{k+1}\right)+\omega_{k+2} \Pi_{{\btheta_{k+1}}}\mu_{\btheta_{k+1}}(\bx_{k+1})-\omega_{k+1} \Pi_{{\btheta_{k}}}\mu_{\btheta_{k}}(\bx_{k})\\
    &=\ddot\btheta_k+\omega_{k+1}h(\bm{\theta}_{k}) +\omega_{k+1}\left( \bm{\nu}_{k+1}+\bm{e}_{k+1}\right)\\
    &=\ddot\btheta_k+\omega_{k+1}h(\ddot\btheta_k) +\omega_{k+1}\left( \bm{\tilde\nu}_{k+1}+\bm{e}_{k+1}\right),\\
\end{split}
\end{equation}
where $\bm{\tilde\nu}_{k+1}=\bm{\nu}_{k+1}+h(\btheta_k)-h(\ddot\btheta_k)$. Next, we proceed to verify the conditions in \textbf{C3}.

(I) 
By the martingale difference property of $\{\bm{e_k}\}$ and the compactness assumption \ref{ass2a}, we know that for any $\widetilde\alpha>2$
$$\E[\bm{e}_{k+1}|\mathcal{F}_k]=\bm{0}, \ \ \ \ \ \sup_{k\geq 0}\E[\|\bm{e}_{k+1}\|^{\widetilde\alpha}|\mathcal{F}_k]<\infty. \eqno{(\text{I})}$$

(II) By the definition of $h(\btheta_k)$ in Eq.(\ref{h_i_theta}), we can easily check that $h(\btheta_k)$ is Lipschitz continuous in a neighborhood of $\widehat\btheta_{\star}$. Combining Eq.(\ref{poisson_reg}), we have $\|h(\btheta_k)-h(\ddot\btheta_k)\|=\mathcal{O}(\|\btheta_k-\ddot\btheta_{k}\|)=\mathcal{O}(\|\omega_{k+1} \Pi_{{\btheta_{k}}}\mu_{\btheta_{k}}(\bx_{k})\|)=\mathcal{O}(\omega_{k+1})$. Then $\E[\|\bm{\nu}_{k+1}\|]\leq C\|\btheta_k-\ddot \btheta_k\|+\mathcal{O}(\omega_{k+2})=\mathcal{O}(\omega_{k+1})$ by the step size condition Eq.(\ref{a1}). In what follows, we can verify
$$\E\left[\frac{\|\bm{\tilde\nu}_{k+1}\|^2}{\omega_{k}}\right]\leq 2\E\left[\frac{\|\bm{\nu}_{k+1}\|^2}{\omega_{k}}\right]+2\E\left[\frac{\|h(\btheta_k)-h(\ddot\btheta_k)\|^2}{\omega_k}\right]=\mathcal{O}(\omega_k)\rightarrow 0. \eqno{(\text{II})} $$

(III) For the martingale difference noise $\bm{e}_{k+1}=\mu_{\btheta_k}(\bx_{k+1})-\Pi_{{\btheta_k}}\mu_{\btheta_k}(\bx_{k})$ with mean 0, we have
\begin{equation*}
    \E[\bm{e}_{k+1}\bm{e}_{k+1}^{\top}|\mathcal{F}_k]=\E[\mu_{\btheta_k}(\bx_{k+1})\mu_{\btheta_k}(\bx_{k+1})^{\top}|\mathcal{F}_k]-\Pi_{{\btheta_k}}\mu_{\btheta_k}(\bx_{k})\Pi_{{\btheta_k}}\mu_{\btheta_k}(\bx_{k})^{\top}.
\end{equation*}
We denote $\E[\bm{e}_{k+1}\bm{e}_{k+1}^{\top}|\mathcal{F}_k]$ by a function $f(\bx_{k})$. Applying Lemma \ref{covariance_estimator}, we have
$$\E[\bm{e}_{k+1}\bm{e}_{k+1}^{\top}|\mathcal{F}_k]=\E[f(\bx_{k})]\rightarrow \int f(\bx)\varpi_{\widehat\btheta_{\star}}d\bx=\lim_{k\rightarrow \infty} \E[\bm{e}_{k+1}\bm{e}_{k+1}^{\top}|\mathcal{F}_k]:=\bR, \eqno{(\text{III})}$$
where $\bR:=\bR(\widehat\btheta_{\star})$ and $\bR(\btheta)$ is also equivalent to $\sum_{k=-\infty}^{\infty} \cov_{\btheta}(H(\btheta, \bx_k), H(\btheta, \bx_0))$. 

Having the conditions C1, C2 and C3 verified, we apply Lemma \ref{sufficiency} and have the following weak convergence for $\ddot\btheta_k$
\begin{equation*}
\begin{split}
    \omega_k^{-1/2}(\ddot\btheta_k-\widehat\btheta_{\star})\Rightarrow\mathcal{N}(0, \bSigma),
\end{split}
\end{equation*}
where $\bSigma=\int_0^{\infty} e^{t h_{\btheta_{\star}}}\circ \bR\circ  e^{th^{\top}_{\btheta_{\star}}}dt$ and $h_{\btheta_{\star}}=h_{\btheta}(\widehat\btheta_{\star})+\widehat\xi\bI$, $\widehat\xi=\lim_{k\rightarrow \infty}\frac{\omega_k^{0.5}-\omega_{k+1}^{0.5}}{\omega_k^{1.5}}$.

Considering the definition that $\ddot\btheta_k=\btheta_k+\omega_{k+1} \Pi_{{\btheta_{k}}}\mu_{\btheta_{k}}(\bx_{k})$ and $\E[ \|\Pi_{{\btheta_{k}}}\mu_{\btheta_{k}}(\bx_{k})\|]$ is uniformly bounded by Eq.(\ref{poisson_reg}), we have 
\begin{equation*}
    \omega_k^{1/2}\Pi_{{\btheta_{k}}}\mu_{\btheta_{k}}(\bx_{k})\rightarrow 0 \text{\ \ \ \ \ in\ probability.}
\end{equation*}

By Slutsky's theorem, we eventually have the desired result
\begin{equation*}
   \omega_k^{-1/2}(\btheta_k-\widehat\btheta_{\star})\Rightarrow\mathcal{N}(0, \bSigma).
\end{equation*}
where the step size $\omega_k$ decays with an order $\alpha\in (0.5, 1]$ such that $\omega_k=\mathcal{O}(k^{-\alpha})$. \qed

\newpage

\section{More on experiments}
\label{details_exp}

\subsection{Mode exploration on MNIST via the scalable random-field function}
\label{mnist_appendix}
For the network structure, we follow \cite{Jarrett09} and choose a standard convolutional neural network (CNN). Such a CNN has two convolutional (conv) layers and two fully-connected (FC) layers. The two conv layers has 32 and 64 feature maps, respectively. The FC layers both have 50 hidden nodes and the network has 5 outputs. A large batch size of 2500 is selected to reduce the gradient noise and reduce the stochastic approximation bias. We fix $\zeta=3e4$ and weight decay 25. For simplicity, we choose 100,000 partitions and $\Delta u=10$. The step size follows $\omega_k=\min\{0.01, \frac{1}{k^{0.6}+100}\}$.

\subsection{Simulations of multi-modal distributions}
\label{simulation_appendix}
The target density function is given by $\pi(\bx)\propto \exp(-U(\bx))$, where $\bx=(x_1, x_2)$ and $U(\bx)$ follows $U(\bx) = 0.2 (x_1^2 + x_2^2) - 2(\cos(2\pi x_1) + \cos(2\pi x_2))$. \begin{wrapfigure}{r}{0.35\textwidth}
   \begin{center}
   \vskip -0.2in
     \includegraphics[width=0.35\textwidth]{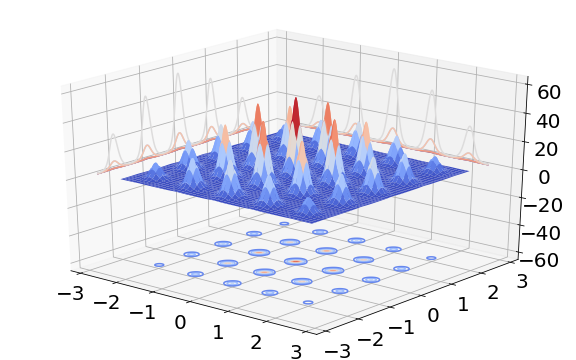}
   \end{center}
   \vskip -0.2in
   \caption{Target density.}
   \label{fig:energy}
\end{wrapfigure}We also include a regularization term $L(x) = \mathbb{I}_{(x_1^2 + x_2^2) > 20} \times ((x_1^2 + x_2^2) - 20)$. This design leads to a highly multi-modal distribution with 25 isolated modes. 
Figure \ref{fig:energy} shows the contour and the 3-D plot of the target density. The ICSGLD and baseline algorithms are applied to this example. For ICSGLD, we set $\epsilon_k=3e^{-3}$, $\tau = 1$, $\zeta=0.75$ and total number of iterations$=8e^4$. Besides, we partition the sample space into 100 subregions with bandwidth $\Delta u=0.125$ and set $\omega_k = \min(3e^{-3}, \frac{1}{k^{0.6}+100})$.

For comparison, we run the baseline algorithms under similar settings. 
For CSGLD, we run a single process 5 times of the time budget and all the settings are the same as those used by ICSGLD. For reSGLD, we run five parallel chains with learning rates $0.001, 0.002, \cdots, 0.005$ and temperatures $1, 2, \cdots, 5$, respectively. We estimate the correction every $100$ iterations. We fix the initial correction 30 and choose the same step size for the stochastic approximation as in ICSGLD. For SGLD, we run five chains in parallel with the learning rate $3e^{-3}$ and a temperature of $1$.
For cycSGLD, we run a single-chain with 5 times of the time budget. We set the initial learning rate as $1e^{-2}$ and choose 10 cycles.
For the particle-based SVGD, we run five chains in parallel. For each chain, we initialize 100 particles as being drawn from a uniform distribution over a rectangle. The learning rate is set to $3e^{-3}$.
\begin{figure}[ht]
\small
\vskip -0.2in
\begin{tabular}{cc}
\includegraphics[height=2.5in,width=2.5in]{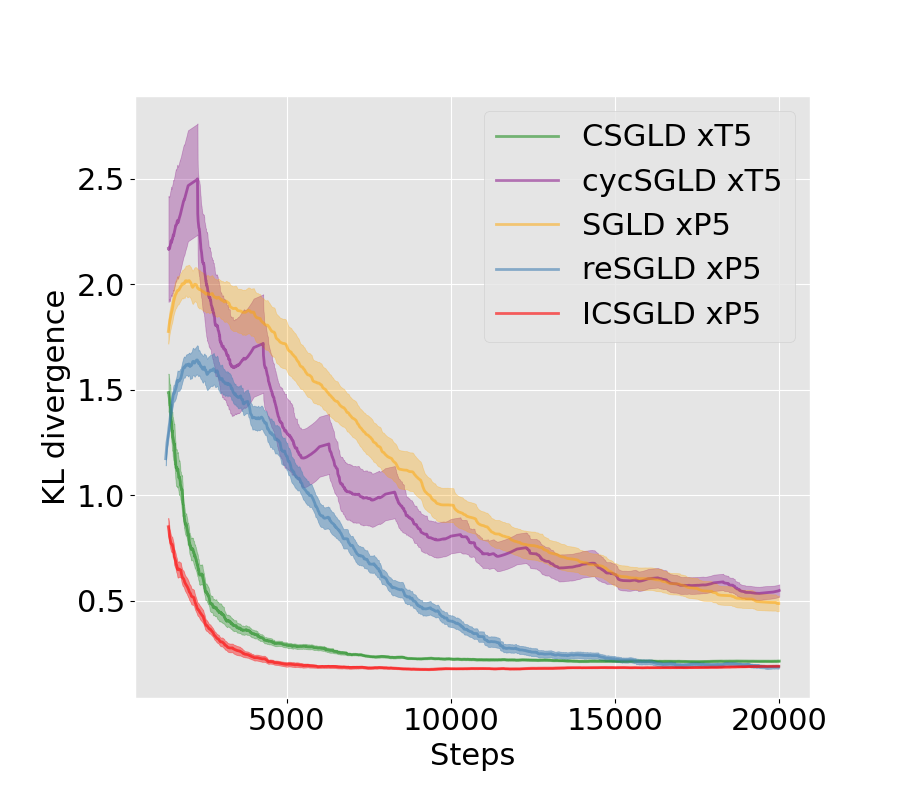} &
\includegraphics[height=2.5in,width=2.5in]{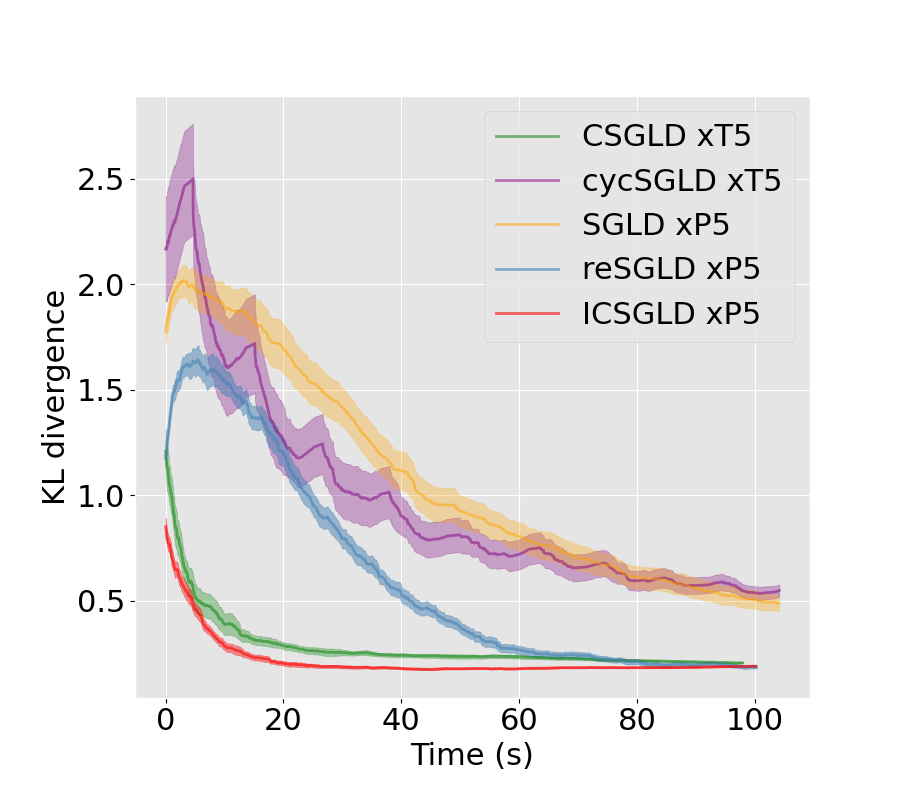} 
\end{tabular}
\vskip -0.1in
  \caption{Estimation KL divergence versus time steps for ICSGLD and baseline methods. We repeat experiments 20 times.}
\label{fig:convergence}
\vskip -0.1in
\end{figure}

To compare the convergence rates in terms of \emph{running steps} and \emph{time} between ICSGLD and other algorithms, we repeat each algorithm 20 times and calculate the mean and standard error over 20 trials. Note that we run all the algorithms based on 5 parallel chains ($\times$P5) except that cycSGLD and CSGLD are run in a single-chain with 5 times of time budget ($\times$T5) and the steps and running time are also scaled accordingly. Figure \ref{fig:convergence} shows that the vanilla SGLD$\times$P5 converges the slowest among the five algorithms due to the lack of mechanism to escape local traps; cycSGLD$\times$T5 slightly alleviates that problem by adopting cyclical learning rates; reSGLD$\times$P5 greatly accelerates the computations by utilizing high-temperature chains for exploration and low-temperature chains for exploitation, but the large correction term inevitably slows down the convergence; ICSGLD$\times$P5 converges faster than all the others and the noisy energy estimators only induce a bias for the latent variables and don't affect the convergence rate significantly. 

For the particle-based SVGD method, since more particles require expensive computations while fewer particles lead to a crude approximation. Therefore, we don't show the convergence of SVGD and only compare the Monte Carlo methods.

\subsection{Deep contextual bandits on mushroom tasks}
\label{bandit_mushroom}
For the UCI Mushroom data set, each mushroom is either edible or poisonous. Eating an edible mushroom yields a reward of 5, but eating a poisonous mushroom has a 50\% chance to result in a reward of -35 and a reward of 5 otherwise. Eating nothing results in 0 reward. All the agents use the same architecture. In particular, we fit a two-layer neural network with 100 neurons each and ReLU activation functions. The input of the network is a feature vector with dimension 22 (context) and there are 2 outputs, representing the predicted reward for eating or not eating a mushroom. The mean squared loss is adopted for training the models. We initialize 1024 data points and keep a data buffer of size 4096 as the training proceeds. The size of the mini-batch data is set to 512. To adapt to online scenarios, we train models after every 20 new observations.

We choose one $\epsilon$-greedy policy (EpsGreedy) based on the RMSProp optimizer with a decaying learning rate  \citep{bandits_showdown} as a baseline. Two variational methods, namely stochastic gradient descent with a constant learning rate (ConstSGD) \citep{Mandt} and Monte Carlo Dropout (Dropout) \citep{Gal16b} are compared to approximate the posterior distribution. For the sampling algorithms, we include preconditioned SGLD (pSGLD) \citep{Li16}, preconditioned CSGLD (pCSGLD) \citep{CSGLD}, and preconditioned ICSGLD (pICSGLD). Note that all the algorithms run 4 parallel chains with average outputs ($\times$P4) except that pCSGLD runs a single-chain with 4 times of computational budget ($\times$T4). In particular for the two contour algorithms, we set $\zeta=20$ and choose a constant step size for the stochastic approximation to fit for the time-varying posterior distributions. For more details on the experimental setups, we refer readers to section D in the supplementary material.

We report the experimental setups for each algorithm. Similar to Table 2 of \citet{bandits_showdown}, the inclusion of advanced techniques may change the optimal settings of the hyperparameters. Nevertheless, we try to report the best setups for each individual algorithm. We train each algorithm 2000 steps. We initialize 1024 mushrooms and
keep a data buffer of size 4096 as the training proceeds. For each step, we are given 20 random mushrooms and train the model 16 iterations every step for the parallel algorithms ($\times$P4); we train pCSGLD$\times$T4 64 iterations every step.

EpsGreedy decays the learning rate by a factor of 0.999 every step; by contrast, all the others choose a fixed learning rate. RMSprop adopts a regularizer of $0.001$ and a learning rate of $0.01$ to learn the preconditioners. Dropout proposes a 50\% dropout rate and each subprocess simulates 5 models for predictions. For the two importance sampling (IS) algorithms, we partition the energy space into $m=100$ subregions and set the energy depth $\Delta u$ as 10. We fix the hyperrameter $\zeta=20$. The step sizes for pICSGLD$\times$P4 and pCSGLD$\times$T4 are chosen as 0.03 and 0.006, respectively. A proper regularizer is adopted for the low importance weights. See Table \ref{hyper_TS} for details.

\begin{table*}[ht]
\begin{sc}
\caption[Table caption text]{Details of the experimental setups.} \label{hyper_TS}
\small
\begin{center} 
\begin{tabular}{cccccccc}
\hline
Algorithm & \scriptsize{Learning rate} & T\upshape{emperature} & \scriptsize{RMS\upshape{prop}} & IS & T\upshape{rain} & D\upshape{ropout} & $\epsilon$-G\upshape{reedy} \\
\hline
EpsGreedy$\times$P4 & 5\upshape{e}-7 (0.999) & 0 & Yes & No & 16 & No & 0.3\%  \\
ConstSGD$\times$P4 & 1\upshape{e}-6 & 0 & No & No & 16 & No & No \\
Dropout$\times$P4 & 1\upshape{e}-6 & 0 & No & No & 16 & Yes (50\%) & No \\
pCSGLD$\times$T4 & 5\upshape{e}-8 & 0.3 & Yes & Yes & 64 & No & No \\
pSGLD$\times$P4 & 3\upshape{e}-7 & 0.3 & Yes & No & 16 & No & No \\
pICSGLD$\times$P4 & 3\upshape{e}-7 & 0.3 & Yes & Yes  & 16 & No & No \\
\hline
\end{tabular}
\end{center}
\end{sc}
\vspace{-0.2in}
\end{table*}

\subsection{Uncertainty estimation}
\label{UQ_appendix}
All the algorithms, excluding M-SGD$\times$P4, choose a temperature of 0.0003 \footnote[2]{We use various data augmentation techniques, such as random flipping, cropping, and random erasing \citep{Zhong17}. This leads to a much more concentrated posterior and requires a very low temperature.}. We run the parallel algorithms 500 epochs ($\times$P4) and run the single-chain algorithms 2000 epochs ($\times$T4). The initial learning rate is 2e-6 (Bayesian settings), which corresponds to the standard 0.1 for averaged data likelihood.

We train cycSGHMC$\times$T4 and MultiSWAG$\times$T4 based on the cosine learning rates with 10 cycles. The learning rate in the last 15\% of each cycle is fixed at a constant value. MultiSWAG simulates 10 random models at the end of each cycle. M-SGD$\times$P4 follows the same cosine learning rate strategy with one cycle.

reSGHMC$\times$P4 proposes swaps between neighboring chains and requires a fixed correction of 4000 for ResNet20, 32, and 56 and a correction of 1000 for WRN-16-8. The learning rate is annealed at 250 and 375 epochs with a factor of 0.2. ICSGHMC$\times$P4 also applies the same learning rate. We choose $m=200$ and $\Delta u=200$ for ResNet20, 32, and 56 and $\Delta u=60$ for WRN-16-8. Proper regularizations may be applied to the importance weights and gradient multipliers for training deep neural networks.

Variance reduction \citep{deng_VR} only applies to reSGHMC$\times$P4 and ICSGHMC$\times$P4 because they are the only two algorithms that require accurate estimations of the energy. We only update control variates every 2 epochs in the last 100 epochs, which maintain a reasonable training time and a higher reduction of variance due to a small learning rate. Other algorithms yield a worse performance when variance reduction is applied to the gradients.

\subsection{Empirical Validation of Reduced Variance}

To compare the $\theta$'s learned from ICSGLD and CSGLD, we try to simulate from a Gaussian mixture distribution $0.4 N(-6, 1) + 0.6 N(4, 1)$, where $N(u, v)$ denotes a Gaussian distribution with mean $u$ and standard deviation $v$. We fix $\zeta=0.9$ and $\Delta u=1$. We run ICSGLD with 1,000,000 iterations \begin{wrapfigure}{r}{0.38\textwidth}
   \begin{center}
   \vskip -0.2in
     \includegraphics[width=0.38\textwidth]{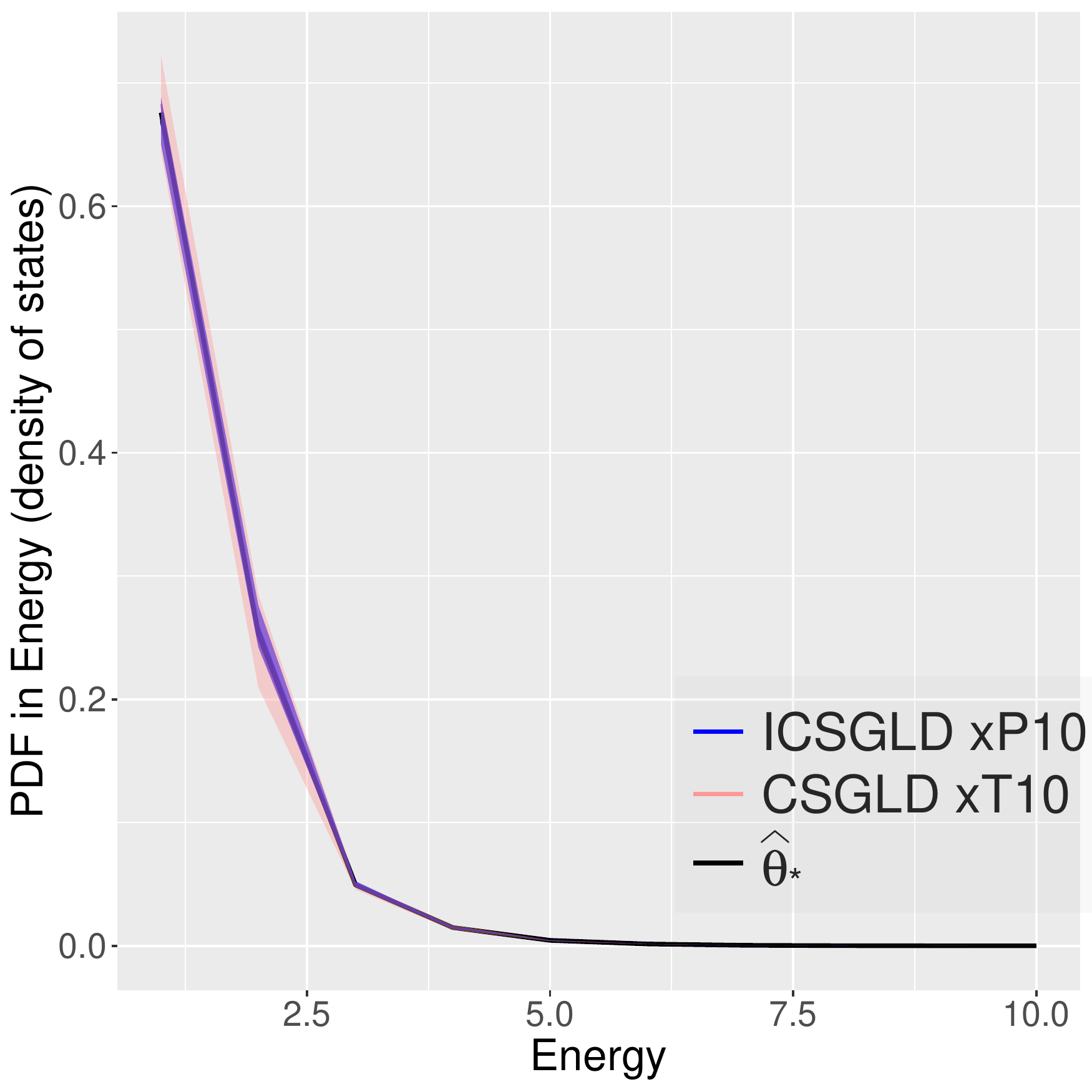}
   \end{center}
   \vskip -0.2in
   \caption{ICSGLD v.s. CSGLD.}
   \label{fig:reduced_variance}
\end{wrapfigure}  based on 10 interacting parallel chains and run CSGLD with 10,000,000 iterations using a single chain. We refer to them as ICSGLD$\times$P10 and CSGLD$\times$T10, respectively. The rest of the settings follows from the experimental setup in section 4.1 \citep{deng2020}. 

To measure the variance of the estimates, we repeated the experiments 10 times and present the mean and two standard deviations for both CSGLD$\times$T10 and ICSGLD$\times$P10 in Figure \ref{fig:reduced_variance}. The results indicate that both estimates of $\theta^{\zeta}$ (by CSGLD and ICSGLD) converge to the equilibrium that approximates the ground truth of the density of states. Notably, ICSGLD$\times$P10 yields a \emph{significantly smaller variance} than CSGLD$\times$T10, but with the same computational budget. This shows the clear advantage of ICSGLD (many interacting short runs) over CSGLD (a single long run) in tackling the \emph{large variance issue} for importance sampling.

\end{document}